\documentclass{colt2020} 

\title[Last iterate is slower than averaged iterate]{Last Iterate is Slower than Averaged Iterate in Smooth Convex-Concave Saddle Point Problems}
\usepackage{times}

\def\bbzero{{\ensuremath{\mathbf 0}}}
\def\bbA{{\ensuremath{\mathbf A}}}

\def\bbv{{\ensuremath{\mathbf v}}}

\def\bbx{{\ensuremath{\mathbf x}}}
\def\bby{{\ensuremath{\mathbf y}}}
\def\bbz{{\ensuremath{\mathbf z}}}
\def\bb0{{\ensuremath{\mathbf 0}}}
\newcommand{\nc}{\newcommand}
\nc{\DMO}{\DeclareMathOperator}
\newcount\Comments  
\renewcommand{\todo}[1]{\ifnum\Comments=1 {\color{red}  [TODO: #1]}\fi}
\nc{\lng}{\langle}
\nc{\rng}{\rangle}
\DMO{\Reg}{Reg}
\DMO{\Ham}{Ham}
\DMO{\Gap}{Gap}
\DMO{\GD}{GD}
\DMO{\GDA}{GDA}
\DMO{\EG}{EG}
\DMO{\OGDA}{OGDA}

\nc{\algnst}[1]{\begin{align*}#1\end{align*}}
\nc{\algn}[1]{\begin{align}#1\end{align}}
\nc{\matx}[1]{\left(\begin{matrix}#1\end{matrix}\right)}
\renewcommand{\^}[1]{^{(#1)}}

\nc{\nuu}{\nu}
\DMO{\diag}{diag}
\DMO{\IC}{IC}
\nc{\pp}{1}
\nc{\MF}{\mathcal{F}}
\DMO{\bil}{bil}
\nc{\MFbil}{\MF^{\bil}}
\nc{\MFbiln}{\MF^{\bil}_{n}}
\nc{\ML}{\mathcal{L}}
\DMO{\funct}{Func}
\nc{\MLfunc}{\ML^{\funct}}
\nc{\MLham}{\ML^{\Ham}}
\nc{\MLgap}{\ML^{\Gap}}
\nc{\Ball}{\MB}

\nc{\bv}{\mathbf{v}}
\nc{\bX}{\mathbf{X}}
\nc{\bY}{\mathbf{Y}}
\nc{\bG}{\mathbf{G}}
\nc{\bg}{\mathbf{g}}
\nc{\bz}{\mathbf{z}}
\nc{\bw}{\mathbf{w}}
\nc{\bB}{\mathbf{B}}
\nc{\bA}{\mathbf{A}}
\nc{\bU}{\mathbf{U}}
\nc{\bJ}{\mathbf{J}}
\nc{\bK}{\mathbf{K}}
\renewcommand{\bb}{\mathbf{b}} 
\nc{\ba}{\mathbf{a}}
\nc{\bc}{\mathbf{c}}
\nc{\bC}{\mathbf{C}}
\nc{\BR}{\mathbb R}
\nc{\BA}{\mathbb{A}}
\nc{\BP}{\mathbb{P}}
\nc{\BC}{\mathbb C}
\nc{\bx}{\mathbf{x}}
\nc{\bS}{\mathbf{S}}
\nc{\bM}{\mathbf{M}}
\nc{\bR}{\mathbf{R}}
\nc{\bN}{\mathbf{N}}
\nc{\by}{\mathbf{y}}
  
\nc{\MO}{\mathcal O}
\nc{\MU}{\mathcal{U}}
\nc{\ME}{\mathcal{E}}
\nc{\MN}{\mathcal{N}}
\nc{\MK}{\mathcal{K}}
\nc{\MS}{\mathcal{S}}
\nc{\MT}{\mathcal{T}}
\nc{\BF}{\mathbb F}
\nc{\BQ}{\mathbb Q}
\nc{\MX}{\mathcal{X}}
\nc{\MA}{\mathcal{A}}
\nc{\MD}{\mathcal{D}}
\nc{\MB}{\mathcal{B}}
\nc{\MZ}{\mathcal{Z}}
\nc{\MY}{\mathcal{Y}}
\nc{\BZ}{\mathbb Z}
\nc{\BN}{\mathbb N}
\nc{\ep}{\epsilon}
\nc{\BH}{\mathbb H}
\nc{\BG}{\mathbb{G}}
\nc{\D}{\Delta}
\nc{\One}{\mathbbm{1}}
\nc{\bOne}{\mathbf{1}}

\nc{\SP}{\mathsf P}
\nc{\SQ}{\mathsf Q}

\nc{\DO}{\accentset{\circ}{\D}}
\nc{\mf}{\mathfrak}
\nc{\mfp}{\mathfrak{p}}
\nc{\mfq}{\mf{q}}
\nc{\Sp}{\mbox{Spec}}
\nc{\Spm}{\mbox{Specm}}
\nc{\hookuparrow}{\mathrel{\rotatebox[origin=c]{90}{$\hookrightarrow$}}}
\nc{\hookdownarrow}{\mathrel{\rotatebox[origin=c]{-90}{$\hookrightarrow$}}}
\nc{\hra}{\hookrightarrow}
\nc{\tra}{\twoheadrightarrow}
\nc{\sgn}{{\rm sgn}}
\nc{\aut}{{\rm Aut}}
\nc{\Hom}{{\rm Hom}}
\nc{\img}{{\rm Im}}
\DMO{\id}{Id}
\DMO{\supp}{supp}
\DMO{\KL}{KL}
\DMO{\BSS}{BSS}
\DMO{\BES}{BES}
\DMO{\BGS}{BGS}
\DMO{\poly}{poly}
\nc{\indep}{\perp}

\nc{\p}{\mathbb{P}}

\nc{\E}{\mathbb{E}}
\nc{\ra}{\rightarrow}
\renewcommand{\t}{\top}

\newtheorem{assumption}{\hspace{0pt}\bf Assumption}

\coltauthor{\Name{Noah Golowich}\thanks{Supported by a Fannie \& John Hertz Foundation Fellowship, an MIT Akamai Fellowship, and an NSF Graduate Fellowship.} \Email{nzg@mit.edu}\\
   \addr Department of EECS, Massachusetts Institute of Technology \AND
  \Name{Sarath Pattathil} \Email{sarathp@mit.edu} \\
   \addr Department of EECS, Massachusetts Institute of Technology \AND
  \Name{Constantinos Daskalakis}\thanks{Supported by NSF Awards IIS-1741137, CCF-1617730 and CCF-1901292, by a Simons Investigator Award, by the DOE PhILMs project (No. DE-AC05-76RL01830), by the DARPA award HR00111990021, by a Google Faculty award, and by the MIT Frank Quick Faculty Research and Innovation Fellowship.} \Email{costis@csail.mit.edu}\\
   \addr Department of EECS, Massachusetts Institute of Technology \AND
  \Name{Asuman Ozdaglar} \Email{asuman@mit.edu}\\
  \addr Department of EECS, Massachusetts Institute of Technology 
}

\begin{document}

\maketitle

\begin{abstract}%
In this paper we study the smooth convex-concave saddle point problem. Specifically, we analyze the last iterate convergence properties of the Extragradient (EG) algorithm. It is well known that the ergodic (averaged) iterates of EG converge at a rate of ${O}(1/T)$ (\cite{nemirovski_prox-method_2004}). In this paper, we show that the last iterate of EG converges at a rate of ${O}(1/\sqrt{T})$. To the best of our knowledge, this is the first paper to provide a convergence rate guarantee for the last iterate of EG for the smooth convex-concave saddle point problem. Moreover, we show that this rate is tight by proving a lower bound of $\Omega(1/\sqrt{T})$ for the last iterate. This lower bound therefore shows a quadratic separation of the convergence rates of ergodic and last iterates in smooth convex-concave saddle point problems.
\end{abstract}


\section{Introduction}
In this paper we study the following saddle-point problem:
\begin{align}
\label{eq:main_problem}
\min_{\bbx \in \mathbb{R}^m} \max_{\bby \in \mathbb{R}^n} f(\bbx, \bby),
\end{align}
where the function $f$ is smooth, convex in $\bbx$, and concave in $\bby$. This problem is equivalent (\cite{facchinei_finite-dimensional_2003}) to finding a global saddle point of the function $f$, i.e., a point $(\bbx^*, \bby^*)$ such that:
\begin{align}
  \label{eq:saddle_point}
f(\bbx^*, \bby) \leq f(\bbx^*, \bby^*) \leq f(\bbx, \bby^*) \ \ \ \forall \ \bbx \in \mathbb{R}^m, \bby \in \mathbb{R}^n.
\end{align}
The saddle point problem (\ref{eq:main_problem}) arises in many fields. Besides its central importance in Game Theory, Online Learning and Convex Programming, it has recently found application in the study of generative adversarial networks (GANS)  (e.g.~\cite{goodfellow_generative_2014,arjovsky_wasserstein_2017}), adversarial examples (e.g.~\cite{madry_towards_2019}), robust optimization (e.g.~\cite{ben-tal_robust_2009}), and reinforcement learning (e.g.~\cite{du_stochastic_2017,dai_sbeed_2018}).

The convex-concave minimax problem~\eqref{eq:main_problem} is a special case of a {\it monotone variational inequality} (see Section \ref{sec:prelim}), which has been studied since the 1960s (\cite{hartman_non-linear_1966,browder_nonlinear_1965,lions_variational_1967,brezis_methodes_1968,sibony_methodes_1970}). Several first-order iterative algorithms to approximate the solution to a monotone variational inequality, including the Proximal Point (PP) algorithm (\cite{martinet_breve_1970,rockafellar_monotone_1976}), the extragradient (EG) algorithm (\cite{korpelevich_extragradient_1976}) and optimistic gradient descent-ascent (OGDA) (\cite{popov1980modification}), have been studied. It is known that the optimal rate of convergence for first-order methods for solving  monotone variational inequalities (and thus (\ref{eq:main_problem})) is $O(1/T)$, and this rate is achieved  by both the EG and OGDA algorithms (\cite{nemirovski_prox-method_2004,mokhtari_convergence_2019, hsieh_convergence_2019, monteiro_complexity_2010,auslender_interior_2005,tseng_accelerated_2008}). However, such convergence guarantees are only known for the {\it averaged (ergodic)} iterates: in particular, if $(\bbx_t, \bby_t)$ are the iterates generated by the EG or OGDA algorithm for the convex-concave problem (\ref{eq:main_problem}), the convergence rate of $O(1/T)$ is known for $(\bar \bx\^T, \bar \by\^T) := (\frac{1}{T} \sum_{t=1}^T\bbx\^t, \frac{1}{T} \sum_{t=1}^T \bby\^t)$.

The EG and OGDA algorithms have additionally received significant recent attention due to their ability to improve the training dynamics in GANs (\cite{chavdarova_reducing_2019,gidel_variational_2018,gidel_negative_2018,liang_interaction_2018,yadav_stabilizing_2017,daskalakis_training_2017}). 
In the saddle point formulation of GANs, given by (\ref{eq:main_problem}), the parameters $\bbx$ and $\bby$ correspond to parameters of the generator and the discriminator, which are usually represented by neural networks, and therefore the function $f$ is {\it not} convex-concave. The goal in such a case is to find a point $(\bx^*, \by^*)$ which satisfies a saddle-point property such as (\ref{eq:saddle_point}) {\it locally}. However, since $f$ is not convex-concave, few, if any, theoretical guarantees are known for the averaged iterates $(\bar \bx_T, \bar \by_T)$; indeed, in practice, the {\it last iterates} $(\bx\^T, \by\^T)$ typically have reasonably good performance. 

Several works including \cite{korpelevich_extragradient_1976,facchinei_finite-dimensional_2003,mertikopoulos_optimistic_2018} prove that, in the convex-concave case, $\lim_{T \ra \infty} (\bx\^T, \by\^T) = (\bx^*, \by^*)$ where $(\bx\^T, \by\^T)$ are the iterates of EG or OGDA, but they do not establish an upper bound on the {\it convergence rate} of the quality of the solution $(\bx\^T, \by\^T)$ to that of $(\bx^*, \by^*)$. Such a convergence rate is known for the {\it best iterate} among $(\bx\^1, \by\^1), \ldots, (\bx\^T, \by\^T)$ for each $T \in \BN$ (\cite{facchinei_finite-dimensional_2003,monteiro_complexity_2010,mertikopoulos_optimistic_2018}), but not on the {\it last iterate} $(\bx\^T, \by\^T)$. Finally, in the case that $f$ is {\it strongly} convex-{\it strongly} concave, linear convergence rates on the distance between the last iterate and the global min-max point (namely, $\| (\bx\^T, \by\^T) - (\bx\^*, \by\^*) \|$) are known (\cite{tseng_linear_1995,gidel_variational_2018,liang_interaction_2018,mokhtari_unified_2019}), but to the best of our knowledge, before our work there were no known convergence rates for the last iterate of EG in the absence of strong convexity. In this paper, we prove the following tight last-iterate convergence guarantees for the EG algorithm in the unbounded setting for different termination criteria including the primal-dual gap and Hamiltonian: 
\begin{theorem}[Last iterate rate for EG; informal version of Theorem \ref{thm:eg_ub}]
  \label{thm:eg_ub_informal}
  The EG algorithm has a last-iterate convergence rate of $O(1/\sqrt{T})$ for monotone variational inequalities satisfying first and second order smoothness; this convergence holds when measured with respect to either the square root of the Hamiltonian (Definition \ref{def:ham}) or the primal-dual gap (Definition \ref{def:pd_gap}). 
\end{theorem}
Theorem \ref{thm:eg_lb_informal} shows that the rate of Theorem \ref{thm:eg_ub_informal} is tight. Moreover, it establishes a quadratic separation between the last iterate of the extragradient algorithm (which converges at a rate of $O(1/\sqrt T)$) and the averaged iterate (which converges at a rate of $O(1/T)$).
\begin{theorem}[Lower bound for 1-SCLIs; informal version of Theorem \ref{thm:1scli_lb}]
  \label{thm:eg_lb_informal}
The $O(1/\sqrt{T})$ last-iterate upper bound of Theorem \ref{thm:eg_ub_informal} is tight for all 1-stationary canonical linear iterative methods (which includes EG; see Definition \ref{def:scli}). 
\end{theorem}

\subsection{Related work}
  \paragraph{Upper bounds on last-iterate convergence rates.}
  Motivated by applications in GANs, several recent papers have focused on proving last-iterate convergence guarantees for various min-max optimization algorithms. Linear convergence rates have been established for EG, OGDA and several of their variants, in the bilinear case, where $f(\bx, \by) = \bx^\t \bM \by + \bb_1^\t \bx + \bb_2^\t \by$ 
  (\cite{daskalakis_training_2017,liang_interaction_2018,gidel_variational_2018, mokhtari_unified_2019, peng_training_2019,zhang_2020_convergence}). \cite{azizian_tight_2019} establishes a similar linear convergence rate for EG, OGDA, and consensus optimization (\cite{mescheder_numerics_2017}) applied to general convex-concave $f$ in the case that a global lower bound of $\gamma > 0$ is known on the singular values of the Jacobian of $\matx{\nabla_\bx f (\bx, \by) \\ -\nabla_\by f (\bx, \by)}$. \cite{daskalakis_last-iterate_2018} study the bilinear case where $\bx,\by$ are constrained to lie in the simplex and show that the iterates of the optimistic hedge algorithm converge to a global saddle point, without providing any rates of convergence.

  \cite{abernethy_last-iterate_2019} proved linear last-iterate convergence rates for Hamiltonian gradient descent when $f$ belongs to a class of `sufficiently bilinear' (possibly nonconvex-nonconcave) problems. Although their result does generalize the strongly convex-strongly concave and bilinear cases, it does not include the full generality of the convex-concave setting; moreover, as it requires computing derivatives of the Hamiltonian $ \| \nabla_\bx f(\bx\^t, \by\^t) \|^2 + \| \nabla_\by f(\bx\^t, \by\^t) \|^2$, it is a {\it second order} method. \cite{hsieh_convergence_2019} proved local linear convergence rates of OGDA to local saddle points in the neighborhood of which $f$ is strongly convex-strongly concave. 
  \cite{azizian2020accelerating} describe a class of convex-concave functions for which first-order algorithms such as EG can be accelerated locally (with linear rates). Several recent works (\cite{gidel_variational_2018,gidel_negative_2018,bailey_finite_2019}) analyze alternating gradient descent-ascent and show that the iterates neither converge or diverge, but rather cycle infinitely in a bounded set. Finally, there are several works (\cite{shamir_stochastic_2012,jain_making_2019,ge_step_2019}) in the literature on {\it non-smooth} convex minimization that compare the convergence of the last iterate and the averaged iterate; the algorithms considered in these papers require decaying step-sizes in order to achieve last-iterate convergence, and so are not directly comparable to our results

\paragraph{Lower bounds.}
Using lower bounds for non-smooth convex minimization (\cite{nemirovsky_information-based_1992}) as a black box,  \cite{nemirovski_prox-method_2004} gives a lower bound of $\Omega(1/T)$ for first-order methods for the smooth convex-concave saddle point problem; this is achieved by, for instance, the EG algorithm with averaged iterates. \cite{ouyang_lower_2019} gave a direct proof of this fact, and extended it to the case where $\bx, \by$ are affinely constrained. The lower bounds of (\cite{nemirovski_prox-method_2004,ouyang_lower_2019}) rely on Krylov subspace techniques, and therefore only apply in the case where $T \leq n$, where $n$ is the dimension of the problem. \cite{azizian_tight_2019,ibrahim_linear_2019} amend this issue of dimension-dependence using the {\it canonical linear iterative (CLI) algorithm} framework of \cite{arjevani_iteration_2016}. The lower bounds in these papers focus primarily on the smooth and strongly-convex strongly-concave case, and proceed by lower bounding the spectral radius of the operator corresponding to a single iteration of a CLI algorithm. Independently \cite{zhang_lower_2019} developed similar lower bounds for the strongly-convex strongly-concave case.

A significant conceptual hurdle in establishing the tight lower bound of $\Omega(1/\sqrt{T})$ in Theorem \ref{thm:eg_lb_informal} is that averaging the iterates of EG produces the asymptotically faster rate of $O(1/T)$. Thus, the framework for our lower bound must rule out such averaging schemes; we do so by proving lower bounds for {\it stationary} CLI (i.e., SCLI) algorithms, i.e., the iterations are time invariant. The class of SCLI algorithms for which our lower bound applies is essentially the same as that of \cite[Theorem 5]{azizian_tight_2019}.

\paragraph{Outline} In Section \ref{sec:prelim} we formally define the problem considered in this paper and introduce some notation. In Section \ref{sec:lb}, we derive a lower bound for the last iterate of 1-SCLI algorithms, of which EG is a special case, establishing Theorem \ref{thm:eg_lb_informal}. 
In Section \ref{sec:ub}, we derive an upper bound for the last iterate of the EG algorithm under first and second-order smoothness assumptions, establishing Theorem \ref{thm:eg_ub_informal}. 

\section{Preliminaries}\label{sec:prelim}
\noindent \textbf{Notation.} Lowercase boldface (e.g., $\bbv$) denotes a vector and uppercase boldface (e.g., $\bbA$) denotes a matrix. We use $\|\bbv\|$ to denote the Euclidean norm of vector $\bbv$. Throughout this paper we will be considering a function $f : \MX \times \MY \ra \BR$, for convex domains $\MX \subseteq \BR^{n_\bx}, \MY \subseteq \BR^{n_\by}$, for some $n_\bx, n_\by \in \BN$. Write $n = n_\bx + n_\by$. We will often write $\MZ := \MX \times \MY$ and $\bz := (\bx, \by)$ as the concatenation of the vectors $\bx, \by$. The gradient of $f$ with respect to $\bbx$ and $\bby$ at $(\bbx_0,\bby_0)$ are denoted by $\nabla_\bbx f(\bbx_0,\bby_0)$ and $\nabla_\bby f(\bbx_0,\bby_0)$, respectively. For a matrix $\bA \in \BR^{n \times n}$, $\| \bA \|_\sigma$ denotes its spectral norm, i.e., the largest singular value of $\bA$. For symmetric matrices $\bA, \bB$, we write  $\bA \preceq \bB$ if $\bB - \bA$ is positive semidefinite (PSD). The {\it diameter} of $\MZ \subset \BR^n$ is $\sup_{\bz, \bz' \in \MZ} \| \bz - \bz'\|$. For a vector $\bz\in \BR^n$ and $D > 0$, let $\Ball(\bz, D)$ denote the Euclidean ball centered at $\bz$ with radius $D$. For a complex number $w \in \BC$, write $\Re(w), \Im(w)$, respectively, to denote the real and imaginary parts of $w$; thus $w = \Re(w) + i\Im(w)$. 

We assume throughout this paper that the function $f(\bbx, \bby)$ is twice differentiable.
To the function $f : \MZ \ra \BR$ we associate an operator $F_f : \MZ \ra \BR^n$, defined by $F_f(\bx, \by) := \matx{\nabla_\bx f(\bx, \by) \\ - \nabla_\by f(\bx, \by)}$. We usually omit the subscript when the function $f$ is clear. It is well-known (\cite{facchinei_finite-dimensional_2003}) that if $f$ is convex-concave, then $F$ is {\it monotone}, meaning that for all $\bz, \bz' \in \MZ$, we have $\lng F(\bz) - F(\bz') , \bz - \bz' \rng \geq 0$. In this case, it is well-known (\cite{facchinei_finite-dimensional_2003}) that a point $\bz^* = (\bx^*, \by^*) \in \MZ$ satisfies the global saddle point property (\ref{eq:saddle_point}) if and only if
\begin{equation}
  \label{eq:vi}
  \langle F(\bz^*), \bz - \bz^* \rangle \geq 0 \quad \forall \bz \in \MZ.
\end{equation}
Finding a point $\bz^*$ satisfying (\ref{eq:vi}) is known as the {\it variational inequality} problem corresponding to $F$. 
To measure the quality of a solution $\bz = (\bx, \by)$ for the saddle point problem (\ref{eq:main_problem}) or equivalently the variational inequality (\ref{eq:vi}) given by a function $f$, two measures are typically used in the literature (see, e.g., \cite{nemirovski_prox-method_2004,monteiro_complexity_2010,mokhtari_convergence_2019}). The first is the {\it Hamiltonian}, which is equal to the squared norm of the gradient of $f$ at $(\bx, \by)$.
\begin{definition}[Hamiltonian]
  \label{def:ham}
For a function $f : \MZ \ra \BR$, the {\bf Hamiltonian}\footnote{Often there is an additional factor of $\frac 12$ multiplying $\| F_f(\bz) \|^2$ in the definition of the Hamiltonian (see, e.g., \cite{abernethy_last-iterate_2019}), but for simplicitly we opt to drop this factor. We do not use any physical interpretation of the Hamiltonian in this paper.} of $f$ at $(\bx, \by) \in \MZ$ is:
\begin{align}
\Ham_f(\bx, \by) := \| \nabla_\bx f(\bx, \by) \|^2 + \| \nabla_\by f(\bx, \by) \|^2 = \| F_f(\bz) \|^2.\nonumber
\end{align}
\end{definition}
Note that if $(\bx, \by)$ is a global saddle point of (\ref{eq:main_problem}), then $\Ham_f(\bx, \by) = 0$. 

The second quality measure of $(\bx, \by)$ is the {\it primal-dual gap}, which measures the amount by which $\by$ fails to maximize $f(\bx, \cdot)$ and by which $\bx$ fails to minimize $f(\cdot, \by)$. 
\begin{definition}[Primal-Dual Gap]
  \label{def:pd_gap}
  For $f : \MZ\ra \BR$, and some convex region $\MX' \times \MY' \subseteq \MZ$, the {\bf primal-dual gap} at $(\bx, \by) \in \MZ$ with respect to $\MX' \times \MY'$ is:
  \begin{align}
    \label{eq:pd_gap}
  \Gap_f^{\MX' \times \MY'}(\bbx,\bby) = \max_{\bby' \in \MY'} f(\bbx, \bby') - \min_{\bbx' \in \MX'} f(\bbx', \bby).
  \end{align}
  When the set $\MX' \times \MY'$ is clear from context, we shall write $\Gap_f(\bbx, \bby)$. 
\end{definition}
As we work in the unconstrained setting, usually we will have $\MZ = \BR^n$. In such a case, we cannot obtain meaningful guarantees on the primal-dual gap with respect to the set $\MX' \times \MY' = \MZ = \BR^n$, since the gap may be infinite, if, for instance, $f$ is bilinear. Thus, in the unconstrained setting, it is necessary to restrict $\MX' \times \MY'$ to be a compact set; following (\cite{mokhtari_convergence_2019}), for our upper bounds, we will usually consider the primal-dual gap with respect to the set $\MX' \times \MY'$ for $\MX' = \Ball(\bx^*, D), \MY' = \Ball(\by^*, D)$ for some $D > 0$. 
As highlighted in (\cite{mokhtari_convergence_2019}), the iterates $(\bx\^t, \by\^t)$ of many convergent first-order algorithms, including EG and PP, lie in $\Ball(\bx^*, O(\| \bx\^0 - \bx^* \|)) \times \Ball(\by^*, O(\| \by\^0 - \by^* \|))$. Thus, choosing $D = O(\| \bx^* - \bx\^0\| + \| \by^* - \by\^0 \|)$ ensures that the set $\MX' \times \MY'$ contains the convex hull of all the iterates $(\bx\^t, \by\^t)$. 

\section{Lower bound for first-order $1$-SCLI algorithms}\label{sec:lb}
In this section we prove lower bounds for the convergence of a broad range of first order algorithms including the EG algorithm for the convex-concave problem saddle point problem (\ref{eq:main_problem}). The class of ``hard functions'' we use to prove our lower bounds are simply bilinear (and thus convex-concave) functions of the form:
\begin{equation}
  \label{eq:quadratic_f}
  f(\bx, \by) = 
  \bx^\t \bM \by + \bb_1^\t \bx + \bb_2^\t \by,
\end{equation}
where $\bb_1, \bb_2, \bx, \by \in \BR^{n/2}$ for some even $n \in \BN$, 
and $\bM \in \BR^{n/2 \times n/2}$ is a square matrix. Then the monotone operator $F = F_f : \BR^n \ra \BR^n$ corresponding to $f$ is of the form
\begin{equation}
  \label{eq:linear_F}
F(\bz) = \bA \bz + \bb \quad \text{where} \quad \bz = \matx{\bx \\ \by}, \bA = \matx{\bbzero & \bM \\ -\bM^\t & \bbzero}, \bb = \matx{\bb_1 \\ -\bb_2}.
\end{equation}
{\bf Remark.} We will assume that the first iterate $\bz\^0$ of all 1-SCLIs considered in this paper is $\bbzero \in \BR^n$; this assumption is without loss of generality, since we can modify $f$ by applying a translation of $\bx, \by$ in (\ref{eq:quadratic_f}) to make this assumption hold for any given $\MA$. 

For $L, D > 0$, we denote the set of $L$-Lipschitz operators $F$ of the form in (\ref{eq:linear_F}) for which $\bM$, and therefore, $\bA$, is of full rank, and for which $ \| \bA^{-1} \bb\| = D$, by $\MFbil_{n,L,D}$.  The parameter $D$ represents the distance between the initialization (namely, $\bbzero$) and the optimal point $\bz^*$ 
, and also measures the diameter of the balls $\MX, \MY$ with respect to which the primal-dual gap is computed for our lower bounds. As discussed in the previous section, this choice of $\MX, \MY$ is motivated by the fact that for   many convergent algorithms such as EG and PP, the iterates never leave $\MX, \MY$. (We also use the same convention for our upper bounds.) 
For $F \in \MFbil_{n,L,D}$, letting $f : \BR^n \ra \BR$ be such that $F = F_f$, there is a unique global min-max point for $f$, which is given by $\bz^* = -\bA^{-1} \bb$. 

Now we are ready to introduce the class of optimization algorithms we consider, namely $\pp$-stationary canonical linear iterative algorithms:
\begin{definition}[$1$-SCLI algorithms, \cite{arjevani_lower_2015}, Definition 1]
  \label{def:scli}
  An algorithm $\MA$ producing iterates $\bz\^0, \bz\^1, \ldots, \in \BR^n$ with access to a monotone first order oracle $F$ is called a {\bf $\pp$-stationary canonical linear iterative ($\pp$- SCLI)} optimization algorithm\footnote{The ``1'' in ``1-SCLI'' denotes that $\bz\^t$ depends only on the previous iterate $\bz\^{t-1}$.} over $\BR^n$ if when $F(\bz) = \bA \bz + \bb$ for some $\bA \in \BR^{n \times n}, \bb \in \BR^n$, the iterates $\bz\^0, \bz\^1, \ldots$ take the form
  \begin{align}
            \label{eq:scli-definition}
            \bz\^t = \bC_0(\bA) \bz\^{t-1} + \bN(\bA) \bb, \quad t \geq 1,
  \end{align}
  for some mappings $\bC_0, \bN : \BR^{n \times n} \ra \BR^{n \times n}$ and initial vector $\bz\^0 \in \BR^n$.

  When we wish to show the dependence of the iterates $\bz\^t$ on the monotone mapping $F$ of (\ref{eq:linear_F}) explicitly, we shall write $\bz\^t(F)$. 
\end{definition}

Notice that EG with constant step size $\eta > 0$, is a 1-SCLI, as its updates given an operator $F$ of the form in (\ref{eq:linear_F}) are of the form
\begin{equation}
  \label{eq:eg_scli}
\bz\^t =\bz\^{t-1} - \eta (\bA(\bz\^{t-1} - \eta (\bA\bz\^{t-1} + \bb)) + \bb) = (I - (\eta \bA) + (\eta \bA)^2) \bz\^{t-1} - (I - \eta \bA)\eta \bb.
\end{equation}

In contrast to minimization, in which it is natural to measure the quality of the iterates $\bz\^t$ via the function value, there are multiple quality measures, including the Hamiltonian $\Ham_f(\cdot)$ (Definition \ref{def:ham}) and the primal-dual gap $\Gap_f(\cdot)$ (Definition \ref{def:pd_gap}), for the setting of min-max optimization. We will refer to such a quality measure as  {\it loss function}, formalized as a mapping $\ML : \BR^n \ra \BR_{\geq 0}$; note that $\ML$ in general depends on $F$.
\begin{definition}[Iteration complexity, \cite{arjevani_lower_2015}]
  \label{def:ic}
Fix $L, D > 0$, and  let $\MA$ be a 1-SCLI algorithm for the saddle point problem for $f$ as in (\ref{eq:quadratic_f}), whose description may depend on $L, D$. Suppose, for each $F \in \MFbil_{n,L,D}$, $\MA$ produces iterates $\bz\^t(F) \in \BR^n$ and suppose an objective (loss) function $\ML_F : \BR^n \ra \BR_{\geq 0}$ is given. Then the {\bf iteration complexity} of $\MA$ at time $T$ and loss functions $\ML_F$, denoted $\IC_{n,L,D}(\MA, \ML; T)$, is defined as follows:
\begin{equation}
  \label{eq:ic_def}
\IC_{n,L,D}(\MA, \ML; T) := \sup_{F \in \MFbil_{n,L,D}} \left\{ \ML_F(\bz\^T(F)) \right\}.
  \end{equation}
\end{definition}
Definition \ref{def:ic} is slightly different from other definitions of iteration complexity in the literature on convex minimization  (\cite{arjevani_lower_2015,nemirovsky_information-based_1992}), in that $\IC_{n,L,D}(\MA, \ML; T)$ is often replaced with the potentially larger quantity $\sup_{t \geq T} \{ \IC_{n,L,D}(\MA, \ML; t) \}$. 
However, since our goal in this section is to prove {\it lower bounds} on the iteration complexity, our results in terms of (\ref{eq:ic_def}) are stronger than those with this alternative definition of iteration complexity.\footnote{This additional strength of our results rules out an algorithm which achieves small loss at iteration $T$ for any function $F$, but has large loss at some iteration $T' > T$. This additional strength to our lower bound could be useful given the cyclical nature of the iterates of many min-max algorithms.}

Finally, we formalize the following convergence property of 1-SCLIs:
\begin{definition}[Consistency, \cite{arjevani_lower_2015}, Definition 3]
  \label{def:consistency}
A $1$-SCLI optimization algorithm $\MA$ is {\bf consistent with respect to an invertible matrix $\bA$} if for any $\bb \in \BR^n$, the iterates $\bz\^t$ of $\MA$ converge to $-\bA^{-1} \bb$. $\MA$ is called {\bf consistent} if it is consistent with respect to all (full-rank) $\bA$ of the form (\ref{eq:linear_F}). 
\end{definition}

We shall need the following consequence of consistency. 
\begin{lemma}[\cite{arjevani_lower_2015}, Theorem 5]
  \label{lem:consistency}
If a $1$-SCLI optimization algorithm $\MA$ is consistent with respect to $\bA$, then 
\begin{equation}
  \label{eq:consistency_eq}
  \bC_0(\bA)= I + \bN(\bA) \bA.
\end{equation}
  \end{lemma}

  \subsection{1-SCLI lower bound}
In this section we state Theorem \ref{thm:1scli_lb}, which gives a lower bound on the convergence rate of 1-SCLIs for convex-concave functions by considering functions $f$ of the form (\ref{eq:quadratic_f}). 
  \begin{theorem}[Iteration complexity lower bounds]
    \label{thm:1scli_lb}
    Let $\MA$ be a consistent 1-SCLI\footnote{More generally, $\MA$ may be any 1-SCLI so that (\ref{eq:consistency_eq}) holds.} and suppose that the inversion matrix $\bN(\cdot)$ of $\MA$ is a polynomial of degree at most $k-1$ with real-valued coefficients for some $k \in \BN$, and let $L, D > 0$. Then the following iteration complexity lower bounds hold:
    \begin{enumerate}
    \item For $F \in \MFbil_{n,L,D}$, set $\MLham_F(\bz) = \| F(\bz) \|^2$. Then $\IC_{n,L,D}(\MA, \MLham; T) \geq \frac{L^2 D^2}{20 T k^2}$.
      \item For $F \in \MFbil_{n,L,D}$, set $\MLgap_F(\bz) = \sup_{\by' : \| \by' - \by^*\| \leq D} f(\bx, \by') - \inf_{\bx' : \| \bx' - \bx^* \| \leq D} f(\bx', \by)$. Then $\IC_{n,L,D}(\MA, \MLgap; T) \geq \frac{LD^2}{k\sqrt{20T}}$. 
      \item For $F = F_f \in \MFbil_{n,L,D}$, set $\MLfunc_F(\bz) = | f(\bx, \by) - f(\bx^*, \by^*)|$. Then
        $$
        \max\left\{\IC_{n,L,D}(\MA, \MLfunc; T),\IC_{n,L,D}(\MA, \MLfunc; 2T)\right\} \geq \frac{LD^2}{36k \sqrt{T} }.
        $$
    \end{enumerate}
  \end{theorem}
  It will follow from Theorem \ref{thm:eg_ub} that the dependence on $L, D$, and $T$ of the lower bounds in Theorem \ref{thm:1scli_lb} is tight; in Proposition \ref{prop:lb-tight-k}, we show additionally that the inverse linear dependence on $k$ is also tight, at least for $T = 1$.
  
Next  we discuss the assumptions made on $\MA$ in Theorem \ref{thm:1scli_lb}. First we remark that consistency is a standard assumption made in the literature on SCLIs and is satisfied by virtually every SCLI used in practice (see, e.g., \cite{arjevani_lower_2015,azizian_tight_2019,ibrahim_linear_2019}). Moreover, if $\MA$ is not consistent, then a lower bound of $\Omega(1)$ holds on $\sup_{t \geq T} \{\IC_{n,L,D}(\MA, \ML, t)\}$ for $\ML \in \{ \MLham, \MLgap \}$ (though the constant may depend on $\MA$): to see this, let $\bA$ be some full-rank matrix and $\bb \in \BR^n$ be so that the iterates $\bz\^t$ of $\MA$ do not converge to $-\bA^{-1} \bb$. Since $\bA$ is full-rank, neither of $\Ham_f(\bz\^t), \Gap_f(\bz\^t)$ converge to 0. 

  The assumption in Theorem \ref{thm:1scli_lb} that $\bN(\bA)$ is a polynomial in $\bA$ of degree at most $k-1$ is essentially the same as the one made in \cite[Theorem 5]{azizian_tight_2019}, which also studied 1-SCLIs (though in the strongly convex case, deriving linear lower bounds). We remark that {\it some} assumption on $\bN(\bA)$ is necessary, as the choice $\bC_0(\bA) = \bbzero, \bN(\bA) = -\bA^{-1}$ leads to $\bz\^t = -\bA^{-1} \bb = \bz^*$ for all $t \geq 1$. The assumption of the polynomial dependence of $\bN(\bA)$ on $\bA$ may be motivated by the fact that, as noted in \cite{azizian_tight_2019}, it includes many known first order 1-SCLI methods, including:
  \begin{itemize}
  \item $k$-extrapolation methods, in which the single ``extra'' gradient step in EG is replaced by $k \geq 1$ steps (see \cite[Eqn. 13]{azizian_tight_2019});
  \item Cyclic Richardson iterations (\cite{opfer_richardsons_1984}), in which a single update from $\bz\^t$ to $\bz\^{t+1}$ consists of a sequence of $k$ gradient updates with different step-sizes $\eta_1, \ldots, \eta_k$ (so that the step sizes cycle between $\eta_1, \ldots, \eta_k$),
  \end{itemize}
  and combinations of the above with varying step-sizes. 
  In particular, Theorem \ref{thm:1scli_lb} applies to the EG algorithm with constant step size; thus, in light of the fact that the averaged iterates $\bar \bz_T$ of EG have primal-dual gap bounded by $O\left( \frac{D^2L}{T} \right)$ (\cite[Theorem 3]{mokhtari_convergence_2019}),  Theorem \ref{thm:1scli_lb} establishes a quadratic gap (in $T$) in the convergence rate between the averaged and last iterates of EG.\footnote{Note that the upper bounds of \cite{mokhtari_convergence_2019} for EG actually apply to the averages of $\bz_{t+1/2} = \bz_t - \eta F(\bz_t)$ as opposed to the averages of $\bz_t$. This does not cause a problem for the separation since our lower bound on $\Gap_f^\MZ(\bz_T)$ (with $\MZ = \MB(\bx^*, D) \times \MB(\by^*, D)$) can be easily extended to a lower bound on $\Gap_f^\MZ(\bz_{T+1/2})$ as long as $\eta < 1/L$ by noting that for $F_f$ $L$-smooth, $\| F_f(\bz_{T+1/2})\| = \| F_f(\bz_T - \eta F_f(\bz_T)) \| \geq (1 - \eta L) \| F_f(\bz_T)\|$, and for the functions $f$ used in the proof of Theorem \ref{thm:1scli_lb} (see (\ref{eq:quadratic_f})), we have $\Gap_f^\MZ(\bz) = D \| F_f(\bz)\|$ for all $\bz \in \BR^n$.} 
  Below we provide the proof of item 1 of Theorem \ref{thm:1scli_lb}; the proofs of items 2 and 3 are deferred to Appendix \ref{sec:lb_proof}.
  \begin{proof}(of item 1 of Theorem \ref{thm:1scli_lb})
  We claim that for all $t \geq 0$,
    \begin{equation}
  \label{eq:zt_explicit}
  \bz\^t  = 
  (\bC_0(\bA)^t - I) \cdot \bA^{-1} \bb.
\end{equation}
To see that (\ref{eq:zt_explicit}) holds, we argue by induction. The base case is trivial since $\bz\^0 = \bbzero$. For the inductive hypothesis, note that
  \begin{align*}
  \bz\^{t+1}  =& \bC_0(\bA) \cdot (\bC_0(\bA)^t - I) \cdot \bA^{-1} \bb + \bN(\bA) \bb \\
  =& \bC_0(\bA) \cdot (\bC_0(\bA)^t - I) \cdot \bA^{-1} \bb + (\bC_0(\bA) - I) \cdot \bA^{-1} \bb = (\bC_0(\bA)^{t+1} - I) \cdot \bA^{-1} \bb ,
\end{align*}
 where the second equality uses consistency of $\MA$ and Lemma \ref{lem:consistency}.

From (\ref{eq:zt_explicit}) it follows that
\begin{align}  \label{eq:use_1_commute}
  \Ham_f(\bz\^t) = 
  \| \bA (\bC_0(\bA)^t - I) \bA^{-1} \bb + \bb \|^2 = \| \bA \bC_0(\bA)^t \bA^{-1} \bb \|^2 = \| \bC_0(\bA)^t \bb \|^2,
\end{align}
where (\ref{eq:use_1_commute}) follows from the fact that $\bC_0(\bA)$ is a polynomial in $\bA$ with scalar coefficients, and therefore $\bA$ and $\bC_0(\bA)$ commute.

Next we describe the choice of $\bA, \bb$: given a dimension $n \in \BN$, Lipschitz constant $L > 0$ and a diameter parameter $D > 0$, for some $\nu \in (0,L)$ (to be specified later), we set
\begin{align}
  \label{eq:choose_A}
  \bM = \nu I \in \BR^{n/2 \times n/2},& \quad \bb_1 = \bb_2 = \matx{\nu D/\sqrt n \\ \vdots \\ \nu D/\sqrt n}, \quad \bA =& \matx{\bbzero & \bM \\ -\bM^\t & \bbzero}, \bb = \matx{\bb_1 \\ -\bb_2}.
\end{align}
From our choice of $\bA$ and the fact that $\| \bA^{-1} \bb \| = \nu^{-1} \| \bb \|$ for all $\bb \in \BR^n$, it follows from (\ref{eq:use_1_commute}) and $\bz\^0 = \bbzero$ that
\begin{equation}
  \label{eq:ham_ratio}
\frac{\Ham_f(\bz\^t)}{\| \bz\^0 - \bz^* \|^2} = \frac{\| \bC_0(\bA)^t \bb \|^2}{\|\bA^{-1} \bb \|^2} = \frac{\nu^2 \| \bC_0(\bA)^t \bb \|^2}{\| \bb \|^2}.
\end{equation}
Recall the assumption that $\bN(\bA)$ is a polynomial in $\bA$ of degree $k-1$ with scalar coefficients. Moreover, by consistency, we have $\bC_0(\bA) = I + \bN(\bA) \bA$, so $\bC_0(\bA)$ is a polynomial in $\bA$ of degree $k$ with scalar coefficients. Thus we may write $\bC_0(\bA) = q_{0,0}I + q_{0,1} \cdot \bA + \cdots + q_{0,k} \cdot \bA^k$, where $q_{0,0}, \ldots, q_{0,k} \in \BR$ and $q_{0,0} = 1$. Write
$$
q_0(y) := q_{0,0}  + q_{0,1} y+ \cdots + q_{0,k} y^k
$$
for $y \in \BC$. 
It is easily verified that $\bA$ has $n/2$ eigenvalues equal to $\nu i$ and $n/2$ eigenvalues equal to $-\nu i$. Therefore, by the spectral mapping theorem (see, e.g., \cite[Theorem 4]{lax07}),  $\bC_0(\bA)$ has $n/2$ eigenvalues equal to each of $q_0(\nu i)$ and $q_0(-\nu i) = \overline{q_0(\nu i)}$. Notice that our choice of $\bA$ in (\ref{eq:choose_A}) is normal;\footnote{A matrix $\bA$ is {\it normal} if and only if there exists a unitary matrix $\bU$ so that $\bU \bA \bU^*$ is diagonal. It is known that if $\bA$ is normal, then the magnitudes of its eigenvalues are equal to its singular values.} hence $\bC_0(\bA)$ is normal as well, meaning the magnitudes of its eigenvalues are equal to its singular values. In particular, all singular values of $\bC_0(\bA)$ are equal to $| q_0(\nu i)|$. Thus, for any vector $\bb \in \BR^n$, $\| \bC_0(\bA) \cdot \bb \| = |q_0(\nu i)| \cdot \| \bb \|$. It follows that
\begin{align}
  \sup_{\nu \in (0,L]}\frac{\nu^2 \| \bC_0(\bA)^t \bb \|^2}{\| \bb \|^2} =& \sup_{\nu \in (0,L]}\nu^2 |q_0(\nu i) |^{2t}\nonumber\\
  \label{eq:del_imaginary}
  \geq& \sup_{\nu \in (0,L]} \nu^2 \left| \sum_{0 \leq k' \leq \lfloor k/2 \rfloor} (-1)^{k'} q_{0,2k'} \cdot \nu^{2k'} \right|^{2t}\\
  = & \sup_{y \in (0,L^2]} y \cdot \left| \sum_{0 \leq k' \leq \lfloor k/2 \rfloor} (-1)^{k'} q_{0,2k'} \cdot y^{k'} \right|^{2t} \nonumber\\
  \label{eq:use_k2}
  > & \frac{L^2}{20tk^2},
\end{align}
where (\ref{eq:use_k2}) follows from Lemma \ref{lem:k2} (see Section \ref{sec:lb_supp_lemma}). 
 The desired bound in item 1 of the theorem statement follows from (\ref{eq:ham_ratio}) with $t=T$ and the fact that $\| \bA^{-1} \bb \| = D$.
  \end{proof}

\section{Upper bound for extragradient}\label{sec:ub}

In this section, we discuss upper bounds for the last iterate of the Extragradient (EG) algorithm. The updates of EG algorithm can be written as:
\begin{align*}
\bbx^{(t + 1)} = \bbx^{(t)} - \eta \nabla_{\bbx} f(\bbx^{(t+1/2)}, \bby^{(t+1/2)}), \qquad
\bby^{(t + 1)} = \bby^{(t)} + \eta \nabla_{\bby} f(\bbx^{(t+1/2)}, \bby^{(t+1/2)}) 
\end{align*}
where
\begin{align*}
\bbx^{(t + 1/2)} = \bbx^{(t)} - \eta \nabla_{\bbx} f(\bbx^{(t)}, \bby^{(t)}), \qquad
\bby^{(t + 1/2)} = \bby^{(t)} + \eta \nabla_{\bby} f(\bbx^{(t)}, \bby^{(t)}) 
\end{align*}
This algorithm can be succinctly written in terms of the operator $F = F_f$, and the concatenated vector $\bbz = (\bbx, \bby)$ as:
\begin{align*}
\bbz^{(t + 1/2)} = \bbz^{(t)} - \eta F(\bbz^{(t)}) \qquad \qquad \bbz^{(t + 1)} &= \bbz^{(t)} - \eta F(\bbz^{(t+1/2)})
\end{align*}

Let $\partial F \in \BR^{n \times n}$ denote the matrix of partial derivatives of $F$; in particular, $(\partial F)_{i,j} = \frac{\partial F_i(\bz)}{\partial \bz_j}$. 
Our upper bound on convergence rates makes use of the following two assumptions, namely of the Lipschitzness of $F$ and $\partial F$:
\begin{assumption}
  \label{asm:lip_grad}
  For some $L > 0$, the operator $F$ is {\bf $L$-Lipschitz}, i.e., for all $\bz, \bz' \in \MZ$, we have that $\| F(\bz) - F(\bz')\| \leq L \| \bz - \bz'\|$. 
\end{assumption}
In the case that $F = F_f$, the assumption that $F$ is $L$-Lipschitz is simply a smoothness assumption on $f$.

\begin{assumption}
  \label{asm:lip_hes}
For some $\Lambda > 0$, the operator $F$ has a {\bf $\Lambda$-Lipschitz derivative}, i.e., for all $\bz, \bz' \in \MZ$, we have that
  $
\| \partial F(\bz) - \partial F(\bz') \|_\sigma \leq \Lambda \| \bz - \bz'\|.
  $
\end{assumption}
Assumption \ref{asm:lip_hes} is standard in the literature on {\it second-order} optimization, both in the minimax setting (see, e.g., \cite[Definition 2.5]{abernethy_last-iterate_2019}) and in the setting of minimization (see, e.g., \cite{nesterov_cubic_2006}). Even for  first-order algorithms, we believe that Assumption \ref{asm:lip_hes} is necessary to obtain a $O(1/\sqrt{T})$ last-iterate convergence rate for convex-concave saddle point optimization, and leave a proof (or disproof) of this fact as an open problem.

In this section our goal is to prove the following theorem.
\begin{theorem}
  \label{thm:eg_ub}
  Suppose $F : \BR^n \ra \BR^n$ is a monotone operator that is $L$-Lipschitz (Assumption \ref{asm:lip_grad})  and has $\Lambda$-Lipschitz derivative (Assumption \ref{asm:lip_hes}). Fix some $\bz\^0 \in \BR^n$, and suppose there is $\bz^* \in \BR^n$ so that $F(\bz^*) = 0$ and $\| \bz^* - \bz\^0\| \leq D$. 
  If the extragradient algorithm with step size $\eta \leq \min \left\{ \frac{5}{ \Lambda D}, \frac{1}{30L} \right\}$ is initialized at $\bz\^0$, then its iterates $\bz\^T$ satisfy
  \begin{equation}
    \label{eq:ham_T_ub}
\| F(\bz\^T) \| \leq \frac{2D}{\eta \sqrt{T}}.
\end{equation}
If moreover $\MZ = \MB(\bx^*, D) \times \MB(\by^*, D)$ and $F(\cdot) = F_f(\cdot) = \matx{ \nabla_\bx f(\cdot) \\ - \nabla_\by f(\cdot) }$ for a convex-concave function $f$, then
\begin{equation}
  \label{eq:gap_T_ub}
\Gap_f^\MZ(\bx\^T, \by\^T) = \max_{\by' \in \MB(\by^*, D)} f(\bx\^T, \by') - \min_{\bx' \in \MB(\bx^*, D)} f(\bx', \by\^T) \leq \frac{2\sqrt{2}D^2}{\eta \sqrt{T}}
\end{equation}
for all $T \in \BN$.
\end{theorem}

\subsection{Proximal point algorithm}
Before proving Theorem \ref{thm:eg_ub}, we briefly discuss similar convergence bounds for an ``idealized'' version of EG, namely the proximal point (PP) algorithm (see \cite{monteiro_complexity_2010, mokhtari_convergence_2019}). The updates of the PP algorithm are given by
$
\bbz^{(t + 1)} = \bbz^{(t)} - \eta F(\bbz^{(t+1)}).
$
As shown in \cite{mokhtari_convergence_2019}, the ergodic iterates of PP and EG have the same rate of convergence (for a constant step size $\eta$); moreover, \cite{mokhtari_unified_2019} showed that the EG algorithm can be viewed as an approximation of the PP algorithm for bilinear functions. It is natural to wonder whether the same rate of $O(1/\sqrt{T})$ of Theorem \ref{thm:eg_ub} applies to the PP algorithm as well. 
This is indeed the case, even without the assumption of $F$ having $\Lambda$-Lipschitz derivatives and $F$ being $L$-Lipschitz. The proof of this (Theorem \ref{thm:pp_ub}) is provided in Appendix \ref{sec:pp_ub}, and it relies on $\| F(\bbz^{(t)} \|$ decreasing monotonically.
\subsection{Proof of Theorem \ref{thm:eg_ub}}
The proof of Theorem \ref{thm:eg_ub} proceeds by first using the well-known fact (\cite{facchinei_finite-dimensional_2003,mertikopoulos_optimistic_2018,mokhtari_convergence_2019}) that for any $T \in \BN$, there is some $t^* \in \{ 1, 2, \ldots, T\}$ so that the $t^*$th iterate $\bz\^{t^*} = (\bx\^{t^*}, \by\^{t^*})$ obtains the upper bound in (\ref{eq:ham_T_ub}), namely that $\| F(\bz\^{t^*})\| \leq \frac{2D}{\eta \sqrt{T}}$ \footnote{This is immediate for the Proximal Point algorithm}; this step relies only on $L$-Lipschitzness of $F$ (Assumption \ref{asm:lip_grad}). The bulk of the proof is then to use Assumption \ref{asm:lip_hes} to show that $\| F(\bz\^t) \|$ does not increase much above $\| F(\bz\^{t^*} \|$ for all $t^* < t \leq T$, from which (\ref{eq:ham_T_ub}) follows. Finally (\ref{eq:gap_T_ub}) is an immediate consequence of (\ref{eq:ham_T_ub}) and the fact that $F$ is convex-concave.

\vspace{0.3cm}
\begin{proof}(of Theorem \ref{thm:eg_ub}).
  Recall that the iterates of the extragradient algorithm are given by 
  $$
\bz\^{t+1/2} = \bz\^t - \eta F(\bz\^t), \quad \quad \bz\^{t+1} = \bz\^t - \eta F(\bz\^{t+1/2}).
  $$

\noindent By Lemma 5(b) in \cite{mokhtari_convergence_2019}, we have that for any $T > 0$,
$$
\sum_{t=0}^{T-1} \eta^2 \| F(\bz\^t) \|^2 = \sum_{t=0}^{T-1} \| \bz\^t - \bz\^{t+1/2} \|^2 \leq \frac{\| \bz_0 - \bz^* \|^2}{1 - \eta^2 L^2} \leq \frac{D^2}{1 - \eta^2 L^2}.
$$
Thus there is some $t^* \in \{0, 1, 2, \ldots, T-1\}$ so that
\begin{equation}
  \label{eq:eg_min_best}
\| F(\bz\^{t^*}) \|^2 \leq \frac{D^2}{T \eta^2 (1-\eta^2 L^2)}.
\end{equation}
Next we show that for each $t \in \{ 1, 2, \ldots, T-1 \}$, $\| F(\bz\^{t+1} \|^2$ is not much greater than $\| F(\bz\^t \|^2$. To do so we need two lemmas; the first, Lemma \ref{lem:ab_exist}, uses Assumption \ref{asm:lip_hes} to write each $F(\bz\^{t+1})$ in terms of $F(\bz\^t)$. 
\begin{lemma}
  \label{lem:ab_exist}
  For all $\bz \in \MZ$, there are some matrices $\bA_\bz, \bB_\bz$ so that $\bA_\bz + \bA_\bz^\t$ and $\bB_\bz + \bB_\bz^\t$ are PSD and
  \begin{equation}
    \label{eq:fztab}
F(\bz - \eta F(\bz - \eta F(\bz))) = F(\bz) - \eta \bA_\bz F(\bz) + \eta^2 \bA_\bz \bB_\bz F(\bz).
\end{equation}
and
\begin{equation}
  \label{eq:ab_conditions}
 \| \bA_\bz - \bB_\bz\|_\sigma \leq \frac{\eta \Lambda}{2} \| F(\bz) - F(\bz - \eta F(\bz)) \|, \quad \| \bA_\bz \|_\sigma \leq L, \quad \| \bB_\bz \|_\sigma \leq L.
\end{equation}
\end{lemma}
The proof of Lemma \ref{lem:ab_exist} is provided in Section \ref{sec:ab_exist_proof}. 
Next, Lemma \ref{lem:ab_diff} will be used to upper bound the norm of the right-hand side of (\ref{eq:fztab}).
\begin{lemma}
  \label{lem:ab_diff}
  Suppose $\bA, \bB \in \BR^{n \times n}$ are matrices so that $\bA + \bA^\t$ and $\bB + \bB^\t$ are PSD and $\| \bA \|_\sigma, \| \bB\|_\sigma \leq 1/30$. Then
  $
\| I - \bA + \bA \bB \|_\sigma \leq \sqrt{1 + 26 \| \bA - \bB \|_\sigma^2}.
  $
\end{lemma}
The proof of Lemma \ref{lem:ab_diff} is deferred to Section \ref{sec:ab_diff_proof}.

By Lemma \ref{lem:ab_exist} and Lemma \ref{lem:ab_diff} with $\bA = \eta \bA_{\bz\^t}, \bB = \eta\bB_{\bz\^t}$, we have that, as long as $\eta < 1/(30 L)$, 
\begin{align*}
  \| F(\bz\^{t+1}) \|^2 & \leq  \| I - \eta \bA_{\bz\^t} + \eta^2 \bA_{\bz\^t} \bB_{\bz\^t} \|_\sigma^2 \cdot \| F(\bz\^t) \|^2 \\
                        & \leq (1 + 26 \eta^2 \| \bA_{\bz\^t} - \bB_{\bz\^t} \|^2) \cdot \| F(\bz\^t) \|^2 \\
                        & \leq (1 + 7\eta^4 \Lambda^2 \cdot \| F(\bz\^t) - F(\bz\^t - \eta F(\bz\^t)) \|^2) \cdot \| F(\bz\^t) \|^2 \\
  \text{($F$ is $L$-Lipschitz)}                 & \leq (1 + 7 \eta^4 \Lambda^2 \cdot \eta^2 L^2 \| F(\bz\^t) \|^2) \cdot \| F(\bz\^ t ) \|^2 \\
                        & \leq (1 + (\eta^4 \Lambda^2 / 100) \cdot \| F(\bz\^t)\|^2) \cdot\| F(\bz\^t) \|^2.
\end{align*}
Next we will prove by induction that for all $t \in \{ t^*, t^* + 1, \ldots, T \}$, we have that $ \| F(\bz\^t) \|^2 \leq \frac{2D^2}{\eta^2 T}$. The base case is immediate by (\ref{eq:eg_min_best}). To see the inductive step, note that if for all $t' \in \{ t^*, \ldots, t\}$, $\| F(\bz\^{t'}) \|^2 \leq \frac{2D^2}{\eta^2 T}$, then
\begin{align*}
  \| F(\bz\^{t+1})\|^2 & \leq \| F(\bz\^t) \|^2 \cdot \left( 1 + \frac{ \Lambda^2 \eta^2 D^2}{50T} \right) \\
  & \leq \| F(\bz\^{t^*}) \|^2 \cdot \left( 1 + \frac{\Lambda^2 \eta^2 D^2}{50T} \right)^{t +1 - t^*} \\
 \text{(since $\eta < 1/(30L)$)} & \leq \frac{D^2}{\eta^2 T (1-1/900)} \cdot \left( 1 + \frac{\Lambda^2 \eta^2 D^2}{50T} \right)^T \leq \frac{2D^2}{\eta^2 T},
\end{align*}
where the last inequality holds as long as $\Lambda^2 \eta^2 D^2/50 \leq \frac 12$, or equivalently, $\eta \leq \frac{5}{ \Lambda D}$. In particular, we get that
$
\| F(\bz\^T) \| \leq \frac{2 D}{\eta \sqrt{T}}.
$
If $F(\bx,\by) =  \matx{ \nabla_\bx f(\bx,\by) \\ -\nabla_\by f(\bx,\by)} $, for some convex-concave function $f$, then, writing $\MX = \MB(\bx^*, D), \MY = \MB(\by^*, D)$, we have
\begin{align*}
  &   \max_{\by' \in \MY} f(\bx\^T, \by')  - \min_{\bx' \in \MX}  f(\bx', \by\^T) \\
  & \leq \max_{\by' \in \MY} \lng \nabla_\by f(\bx\^T, \by\^T), \by' - \by\^T \rng + \max_{\bx' \in \MX} \lng \nabla_\bx f(\bx\^T, \by\^T), \bx\^T - \bx' \rng \\
  & = \max_{\bz' \in \MZ} \lng F(\bz\^T), \bz\^T - \bz' \rng \leq \| F(\bz\^T) \| \cdot D\sqrt{2} \leq \frac{2\sqrt{2}D^2}{\eta \sqrt{T}}.
\end{align*}
\vspace{-1cm}
\end{proof}

\section{Conclusion and Future Work}
In this paper we establish a $O(1/\sqrt{T})$ upper bound on the primal-dual gap for the $T$th iterate of EG, and show that this is tight among 1-SCLI algorithms. This is slower than the primal-dual gap of $O(1/T)$ for the average of the first $T$ iterates of EG (\cite{nemirovski_prox-method_2004}). An interesting direction for future work is to determine if there is a provable benefit to averaging in the nonconvex-nonconcave case. Some experimental work has suggested that such a benefit to averaging exists (\cite{yazici_unusual_2019}); moreover, averaging is effective even for large-scale GANs (\cite{brock_large_2019}).

Another direction for future work is to extend the lower bound of Theorem \ref{thm:1scli_lb} (or prove a stronger upper bound) for algorithms with decaying step-sizes, which correspond to {\it non-stationary} CLIs. Such a question is only nontrivial for the case of 1-CLIs, as the averaged iterates of extragradient can be written as the iterates of a particular 2-CLI\footnote{We refer the reader to Section \ref{sec:2scli-averaged} for the verification of this statement.}, and the $O(DL^2/T)$ rate of convergence of the averaged iterates of extragradient is known to be optimal (\cite{nemirovski_prox-method_2004}). Towards this question, we show in Section \ref{sec:eg-timevarying} that, in contrast to the case for non-smooth convex minimization (\cite{jain_making_2019,shamir_stochastic_2012}), any choice of decaying step-size for the EG algorithm cannot improve the $\Omega(1/\sqrt{T})$ lower bound from Theorem \ref{thm:1scli_lb}.

\acks{We are grateful to Aleksander Madry for helpful suggestions and to anonymous reviewers for helpful comments on the paper.}

\bibliography{minimax_opt_finalpaper}

\appendix

\section{Proofs for Theorem \ref{thm:1scli_lb}}
\subsection{Proof of items 2 and 3 of Theorem \ref{thm:1scli_lb}}
\label{sec:lb_proof}
  \begin{proof}(of items 2 and 3 of Theorem \ref{thm:1scli_lb})
  We begin with item 2, namely the lower bound on the primal-dual gap. The choice of $\bM, \bA, \bb_1, \bb_2$ (which depend on $\nu \in (0,L]$) is exactly the same as for item 1, and is given in (\ref{eq:choose_A}). 
Write $\MZ := \MB(\bx^*, D) \times\MB(\by^*, D)$. Next we compute $\Gap_f^{\MZ}(\bz\^t)$ in a similar manner to the Hamiltonian in (\ref{eq:use_1_commute}). The components of the primal-dual gap $\Gap_f^\MZ(\bx,\by)$ for a given point $(\bx,\by) \in \BR^n$ are given as follows:
  \begin{align*}
    \max_{\by' \in \MY} f(\bx,\by')  &= \bb_1^\t \bx + \max_{\by' : \| \by'- \by^* \| \leq D} \langle \by', \bM^\t \bx + \bb_2 \rangle \\
                                     & = D \| \bM^\t \bx + \bb_2\|  +\lng \by^*, \bM^\t \bx + \bb_2 \rng +  \bb_1^\t \bx \\
                                     & = D \| \bM^\t \bx + \bb_2 \| + \lng -\bM^{-1} \bb_1, \bM^\t \bx  + \bb_2 \rng + \lng \bb_1, \bx \rng \\
                                     & = D \| \bM^\t \bx + \bb_2 \| - \lng \bM^{-1} \bb_1, \bb_2 \rng.\\
    -\min_{\bx' \in \MX} f(\bx',\by) &= - \bb_2^\t \by -\min_{\bx' : \| \bx' - \bx^* \| \leq D} \langle \bx', \bM\by + \bb_1\rangle \\
                                     & = D \| \bM \by + \bb_1 \| - \lng \bx^*, \bM \by + \bb_1 \rng - \bb_2^\t \by \\
                                     & = D \| \bM \by + \bb_1 \| - \lng -(\bM^\t)^{-1} \bb_2, \bM \by + \bb_1 \rng - \lng \bb_2, \by \rng \\
                                     & = D \| \bM \by + \bb_1 \| + \lng \bM^{-1} \bb_1, \bb_2 \rng.
  \end{align*}
  Thus
  \begin{equation}
    \label{eq:gap-ham-D}
\Gap_f^\MZ(\bx, \by) = \max_{\by' \in \MY} f(\bx, \by') - \min_{\bx' \in \MX} f(\bx', \by) = D \| \bM^\t \bx + \bb_2 \| + D \| \bM \by + \bb_1 \| = D \| \bA \bz + \bb \|,
\end{equation}
and so
\begin{equation}
  \label{eq:gap_zt_exact}
\Gap_f^\MZ(\bx\^t, \by\^t) = D \| \bC_0(\bA)^t \bb \|.
\end{equation}

From (\ref{eq:gap_zt_exact}) we have
\begin{equation}
  \label{eq:gap_ratio}
\frac{\Gap_f^\MZ(\bx\^t, \by\^t)}{\| \bz\^0 - \bz^* \|^2} = \frac{D \| \bC_0(\bA)^t \bb \|}{\| \bA^{-1} \bb \|^2} = \frac{\nu \| \bC_0(\bA)^t \bb \|}{ \| \bb \|}.
\end{equation}

 The desired bound in item 2 of the theorem statement follows from (\ref{eq:gap_ratio}), (\ref{eq:use_k2}), and the fact that $\| \bA^{-1} \bb \| = D$.

Next we turn to convergence in function value (item 3 of the theorem). First note that
\begin{align}
  & f(\bx\^t, \by\^t) - f(\bx^*, \by^*) \nonumber\\
  =&  (\bx\^t)^\t \bM \by\^t + \lng \bx\^t, \bb_1 \rng + \lng \by\^t, \bb_2 \rng - (\bx^*)^\t \bM \by^* - \lng \bx^*, \bb_1 \rng - \lng \by^*, \bb_2 \rng \nonumber\\
  =& \lng \bx\^t - \bx^*, \bM (\by\^t - \by^*) \rng - 2 \lng \bx^*, \bM \by^* \rng + \lng \bx^*, \bM \by^* \rng + \lng \by^*, \bM^\t \bx^* \rng \nonumber \\ 
  =& \lng \bx\^t - \bx^*, \bM (\by\^t - \by^*) \rng\nonumber\\
  \label{eq:fv_calc}
  =& \nu \cdot \sum_{i=1}^{n/2} (\bx_i\^t - \bx^*_i) \cdot (\by_i\^t - \by^*_i), 
\end{align}
where we have used that $\by^* = -\bM^{-1} \bb_1, \bx^* = -(\bM^\t)^{-1} \bb_2$.

Note that the diagonalization of $\bA$ can be written as 
$$
\bA = \bU \cdot \diag(\nu i, \cdots, \nu i, -\nu i, \cdots -\nu i) \cdot \bU^{-1}, \quad \quad \bU = \frac{1}{\sqrt{2}} \cdot \matx{1 & 0 & \cdots & 0 & 1 & 0 & \cdots & 0 \\
  0 & 1 & \cdots & 0 & 0 & 1 & \cdots & 0 \\
  & & \vdots & & & & \vdots & \\
  i & 0 & \cdots & 0 & -i &  0 & \cdots & 0 \\
  0 & i & \cdots & 0 & 0 & -i & \cdots & 0 \\
  & & \vdots & & & & \vdots & }.  
$$
Since $\bU$ is unitary, it follows from (\ref{eq:zt_explicit}) that
\begin{align*}
&  \bz\^t  - \bz^* \nonumber \\
  =& \bC_0(\bA)^t \bA^{-1}\bb \\
  =& \nu^{-1} \bU \cdot  \diag(q_0(\nu i), \cdots, q_0(\nu i), q_0(-\nu i), \cdots, q_0(-\nu i))^t \cdot  \diag(-i, \cdots, -i, i, \cdots, i) \cdot \bU^{-1} \bb  \\
  =&  \frac{D}{\sqrt{2n}} \cdot \bU \cdot (q_0(\nu i)^t(1-i) , \ldots, q_0(\nu i)^t (1-i), q_0(-\nu i)^t (1+i), \ldots, q_0(-\nu i)^t (1+i))^\t  \\
  =& \frac{D}{\sqrt{n}} \cdot (\Re(q_0(\nu i)^t (1-i)), \ldots,\Re( q_0(\nu i)^t (1-i)),- \Im(q_0(\nu i)^t(1-i)), \ldots,- \Im(q_0(\nu i)^t(1-i))) .
\end{align*}
Now let us write $q_0(\nu i) = |q_0(\nu i)| \cdot e^{i\theta(\nu)}$, where $\theta(\nu) \in [0, 2\pi)$. It folllows from (\ref{eq:fv_calc}) that
\begin{align}
  f(\bx\^t, \by\^t) &- f(\bx^*, \by^*) \nonumber \\
  =&\nu \sum_{i=1}^{n/2} (\bx_i\^t - \bx_i^*) \cdot (\by_i\^t - \by_i^*)\nonumber \\
  =&\nu D^2 \cdot  \frac 12 \cdot |q_0(\nu i)|^{2t} (\cos(t\theta(\nu)) + \sin(t\theta(\nu))) \cdot (\cos(t\theta(\nu)) - \sin(t\theta(\nu)))\nonumber \\
  =& \nu D^2 \cdot \frac{1}{2} \cdot |q_0(\nu i)|^{2t} \cdot \cos(2 t \theta(\nu)) \nonumber\\
  \label{eq:func_final}
  =& \nu D^2\cdot \frac{1}{2} \cdot \Re(q_0(\nu i)^{2t}).
\end{align}
Now fix some $T$. It follows in a manner identical to (\ref{eq:use_k2}), using Lemma \ref{lem:k2}, that there is some $\nu_*$ with $\nu_*^2 \in [L^2/(40Tk^2),L^2]$ so that $ \nu_*^2 \cdot |q_0(\nu_* i)|^{8T} \geq \frac{L^2}{80 Tk^2}$, which implies $\nu_* \cdot |q_0(\nu_* i)|^{4T} \geq \frac{L}{\sqrt{80T} k}$. We claim that also $\nu_* \cdot |q_0(\nu_* i)|^{2T} \geq \frac{L}{\sqrt{80T} k}$. If $|q_0(\nu_* i)| \geq 1$, this is immediate from $\nu_* \geq L/(\sqrt{40T}k)$; otherwise, this follows from $|q_0(\nu_* i)|^{2T} \geq |q_0(\nu_* i)|^{4T}$. To complete the proof we consider two cases:

{\bf Case 1.} If $|\Re(q_0(\nu_* i)^{2T})| \geq \frac 12 \cdot |q_0(\nu_* i)^{2T}|$, then by (\ref{eq:func_final}) $|f(\bx\^T, \by\^T) - f(\bx^*, \by^*)| \geq \frac{LD^2}{\sqrt{1280T}k}$ (where $f$ is so that $\nu$ in (\ref{eq:choose_A}) is set to $\nu_*$), and we get that $\IC_{n,L,D}(\MA, \MLfunc; T) \geq \frac{LD^2}{\sqrt{1280T}k} \geq \frac{LD^2}{36k\sqrt{T}}$.  

{\bf Case 2.} In the other case that $|\Re(q_0(\nu_* i)^{2T})| \leq \frac 12 \cdot |q_0(\nu_* i)^{2T}|$, we have $2T \theta(\nu_*) \in [\pi/3, 2\pi/3] \cup [-2\pi/3, -\pi/3]$. Hence $4T \theta(\nu_*) \in [2\pi/3, 4\pi/3]$
, and so $|\Re(q_0(\nu_* i)^{4T})| \geq \frac 12 \cdot |q_0(\nu_* i)^{4T}|$. By (\ref{eq:func_final}) and the fact that $\nu_* \cdot |q_0(\nu i)|^{4T} \geq \frac{L}{\sqrt{80T}k}$, it follows that in this case we have $\IC_{n,L,D}(\MA, \MLfunc; 2T) \geq \frac{LD^2}{\sqrt{1280T}k}$. 
\end{proof}

\subsection{Supplementary lemmas for Theorem \ref{thm:1scli_lb}}
\label{sec:lb_supp_lemma}
Lemma \ref{lem:k2} below is similar to the bounds derived in \cite[Section 2.3.B]{nemirovsky_information-based_1992}, but it achieves a better dependence on $t$; in particular, if the bounds in \cite[Section 2.3.B]{nemirovsky_information-based_1992} are used in a black-box manner, one would instead get a lower bound of $\Omega(L/t^2k^2)$ in (\ref{eq:k2mu}).
\begin{lemma}
  \label{lem:k2}
  Fix some $k, t \in \BN$, $L > 0$. Let $r(y) \in \BR[y]$ be a polynomial with real-valued coefficients of degree at most $k$, such that $r(0) = 1$. Then
  \begin{equation}
    \label{eq:k2mu}
\sup_{y \in (0,L]} y \cdot |r(y)|^t \geq \sup_{y \in [L/(20tk^2), L]} y \cdot |r(y)|^t > \frac{L}{40tk^2}.
  \end{equation}
\end{lemma}
\begin{proof}
  Set $\mu := \frac{L}{20tk^2}$. Then $\sqrt{L/\mu} - 1 = \sqrt{20tk^2}  -1 \geq \sqrt{12t} \cdot k$. By Lemma \ref{lem:k2_condition} we have that
  $$
\sup_{y \in [\mu, L]} y \cdot |r(y)|^t \geq \frac{L}{20tk^2} \cdot \left(1 - \frac{6k^2}{(\sqrt{L/\mu} - 1)^2} \right)^t \geq \frac{L}{20tk^2} \cdot (1-1/(2t))^t \geq \frac{L}{40tk^2}.
  $$
\end{proof}

\begin{lemma}
  \label{lem:k2_condition}
  Fix some $k \in \BN$ and $L > \mu > 0$ such that $k \leq \sqrt{L/\mu} - 1$. Let $r(y) \in \BR[y]$ be a polynomial with real-valued coefficients of degree at most $k$, such that $r(0) = 1$. Then
  \begin{equation}
\label{eq:k2mul}
\sup_{y \in [\mu,L]} |r(y)| > 1 - \frac{6k^2}{(\sqrt{L/\mu} - 1)^2}.
  \end{equation}
\end{lemma}
Lemma \ref{lem:k2_condition} is very similar to the combination of Lemmas 5 and 12 in \cite{azizian_tight_2019}, but has a superior dependence on $k$. In particular, we could use \cite[Lemmas 5 \& 12]{azizian_tight_2019} to conclude that a lower bound of $1 - k^3 \cdot \frac{4\mu}{L\pi}$ holds in (\ref{eq:k2mul}), which is smaller than $1 - \frac{6k^2}{(\sqrt{L/\mu} - 1)^2}$ for sufficiently large $k$ (e.g., $k > 10$). We also remark that the proof of Lemma \ref{lem:k2_condition} is much simpler than that of \cite[Lemmas 5 \& 12]{azizian_tight_2019}, though the proofs use similar techniques.
\begin{proof}(of Lemma \ref{lem:k2_condition}).
  Let $T_k(y)$ denote the Chebyshev polynomial of the first kind of degree $k$; it is characterised by the property that:
  \begin{equation}
    \label{eq:tk_define}
T_k\left( \cos\left( \frac{j\pi}{k} \right) \right) = (-1)^j, \quad \quad j = 0, 1, \ldots, k,
\end{equation}
which turns out to be equivalent to the property that
\begin{equation}
  \label{eq:tk_z}
T_k\left( \frac{1}{2} \cdot \left( z + \frac 1z \right)\right) = \frac{1}{2} \cdot \left( z^k + \frac{1}{z^k} \right), \quad \quad \forall z \in \BC.
\end{equation}
It follows immediately from (\ref{eq:tk_z}), that for $k$ odd, $T_k$ is an odd function, and for $k$ even, $T_k$ is an even function. 

Let $q(y) = \frac{T_k\left( \frac{2y - (\mu + L)}{L-\mu}\right)}{T_k\left( \frac{L+\mu}{L-\mu}\right)}$. 
Then $q(0) = 1$. Using (\ref{eq:tk_define}) and the fact that $r(0) = q(0) = 1$, it was shown in \cite[Lemma 2]{arjevani_iteration_2016} that
$$
\sup_{y \in [\mu, L]} |r(y)| \geq \sup_{y \in [\mu, L]} |q(y)|.
$$
Let $\kappa = L/\mu$. From (\ref{eq:tk_define}) we have that $\sup_{y \in [\mu, L]} |q(y)| \geq \frac{1}{T_k\left( \frac{L+\mu}{L-\mu}\right)} = \frac{1}{T_k \left(\frac{\kappa + 1}{\kappa - 1}\right)}$ (in fact, equality holds). At this we depart from the proof of \cite[Lemma 2]{arjevani_iteration_2016}, noting simply that a tighter lower bound on $\frac{1}{T_k\left(\frac{\kappa + 1}{\kappa - 1}\right)}$ than the one shown in \cite[Lemma 2]{arjevani_iteration_2016} holds when $k^2 \ll \kappa$.  In particular, since $\frac{\sqrt{\kappa} + 1}{\sqrt{\kappa} - 1} + \frac{\sqrt{\kappa} - 1}{\sqrt{\kappa} + 1} =2 \cdot  \frac{\kappa + 1}{\kappa - 1}$, (\ref{eq:tk_z}) gives that
\begin{align}
  T_k \left( \frac{\kappa + 1}{\kappa - 1} \right) =&   \frac 12 \cdot \left( \left(\frac{\sqrt{\kappa} + 1}{\sqrt{\kappa} - 1}\right)^k + \left(\frac{\sqrt{\kappa} - 1}{\sqrt{\kappa} + 1}\right)^k\right)\nonumber\\
  \label{eq:exp_x2}
  < & \frac{1}{2} \cdot \left( \left(1 + \frac{2k}{\sqrt{\kappa} - 1} + \frac{(2k)^2}{(\sqrt{\kappa} - 1)^2} \right) + \left(1 - \frac{2k}{\sqrt{\kappa} + 1} + \frac{(2k)^2}{(\sqrt{\kappa} + 1)^2}\right)\right) \\
  \leq & \frac{1}{2} \cdot \left( 2 + \frac{4k}{\kappa - 1} + \frac{8k^2}{(\sqrt{\kappa} - 1)^2}\right)\nonumber\\
  \label{eq:tk_final}
  \leq & 1+ \frac{2k + 4k^2}{(\sqrt\kappa - 1)^2}.
\end{align}
Above (\ref{eq:exp_x2}) follows from the fact $k \leq \sqrt{\kappa} - 1$ and that for $-2 \le yk \le 2$, we have that
$$
(1+y)^k \leq \exp(yk) \leq 1 + yk + 2 y^2 k^2.
$$
From (\ref{eq:tk_final}) it follows that $$\frac{1}{T_k\left( \frac{\kappa + 1}{\kappa -1}\right)} > 1 - \frac{2k + 4k^2}{(\sqrt \kappa -1)^2} \geq 1 - \frac{6k^2}{(\sqrt \kappa -1)^2}.
$$
\end{proof}

\subsection{Tightness of dependence on the degree $k$}
The below proposition establishes that the inverse linear dependence on $k$ in Theorem \ref{thm:1scli_lb} is tight:
\begin{proposition}
  \label{prop:lb-tight-k}
 Then there is a consistent 1-SCLI $\MA$ whose inversion matrix $\bN(\cdot)$ is a polynomial of degree at most $k-1$ so that 
  \begin{equation}
    \label{eq:lb-tight-k}
    \IC_{n,L,D}(\MA, \MLham; 1) \leq O \left( \frac{L^2D^2}{k^2}\right), \ \  \max \left\{ \IC_{n,L,D}(\MA, \MLgap; 1), \IC_{n,L,D}(\MA, \MLfunc; 1) \right\} \leq O \left( \frac{LD^2}{k} \right).
  \end{equation}
  Iteration complexities are defined with respect to $\MFbil_{n,L,D}$ as in Definition \ref{def:ic}.
\end{proposition}
{\bf Remark.} We note that the upper bounds in Proposition \ref{prop:lb-tight-k} hold more generally with respect to any monotone linear operator $F(\bz) = \bA \bz + \bb$. We stick with the class $\MFbil_{n,L,D}$ of operators corresponding to a bilinear function $f$ (as in Definition \ref{def:ic}) to simplify notation. 
  \begin{proof}(of Proposition \ref{prop:lb-tight-k})
  Consider a monotone operator $F = F_f \in \MFbil_{n,L,D}$ of the form $F(\bz) = \bA \bz + \bb$. For $t \geq 0$, let $\bw\^t$ be the iterates obtained by running extragradient on the monotone operator $F$ starting at $\bw\^0 = \bbzero$ and with step size $\eta = 1/(2L)$. Letting
$$
\bC_0(\bA) := I - \eta \bA + (\eta\bA)^2, \qquad \bN(\bA) := -\eta (I - \eta \bA) \bb,
$$
by (\ref{eq:eg_scli}), we can write
$$
\bw\^t = (\bC_0(\bA)^t + \bC_0(\bA)^{t-1} + \cdots + \bC_0(\bA) + I) \cdot \bN(\bA) \bb.
$$
For any $T > 0$, denote the averaged iterates up to time $T$ by $\bar \bw\^T := \frac{\bw\^0 + \cdots + \bw\^T}{T+1}$. Also write $\bar \bw\^T = (\hat \bx\^T, \hat \by\^T)$. By \cite[Theorem 3]{mokhtari_convergence_2019}, we have, with $\MX :=\MB(\bx^*, D), \MY := \MB(\by^*, D)$,
$$
\MLgap_F(\bar \bw\^T) = \max_{\by \in \MY} f(\hat \bx\^T, \by) - \min_{\bx \in \MX} f(\bx, \hat \by\^T) \leq O \left( \frac{D^2L}{T} \right).
$$
It follows immediately from the fact that $f$ is convex-concave that also $\MLfunc(\bar\bw\^T) \leq O \left(\frac{D^2 L}{T} \right)$. 
By (\ref{eq:gap-ham-D}) it follows that $\MLham_F(\bar \bw\^T) = \| \bA \bar \bw\^T + \bb \|^2 \leq O \left( \frac{D^2 L^2}{T^2} \right)$. 
   
Next fix $T = \lfloor \frac{k-1}{2} \rfloor$, and define the polynomial
$$
\bN'(\bA) := \frac{(\bC_0(\bA)^T + 2\bC_0(\bA)^{T-1} + \cdots + (T+1)I) \cdot \bN(\bA)}{T+1},
$$
so that $\bN'(\bA)$ is a polynomial in $\bA$ with real-valued coefficients of degree at most $2T + 1 \leq k-1 $. Also let $\bC_0'(\bA) := I + \bN'(\bA) \cdot \bA$. It is immediate from the definition of $\bN'(\cdot)$ that $\bw\^T = \bN'(\bA) \cdot \bb$, and so $\bw\^T$ is the first iterate (namely, $\bz\^1$, as in (\ref{eq:scli-definition})) in the consistent 1-SCLI defined by $\bC_0'(\cdot), \bN'(\cdot)$. 

  \end{proof}

\section{Proofs of Lemmas \ref{lem:ab_exist} and \ref{lem:ab_diff}}
\label{sec:ub_proof}

\subsection{Preliminary lemmas}
Before proving Lemmas \ref{lem:ab_exist} and \ref{lem:ab_diff} we state a few simple lemmas. 
\begin{lemma}(\cite{nesterov_cubic_2006})
  \label{lem:nspsd}
If $\MZ \subset \BR^n$ and $F : \MZ \ra \BR^n$ is monotone, then for any $\bz, \bw \in \BR^n$,
  $
\bz^\t (\partial F(\bw)) \bz \geq 0.
$
Equivalently, $\partial F(\bw) + \partial F(\bw)^\t$ is PSD. 
\end{lemma}
\begin{lemma}
  \label{lem:xy_ineq}
  Let $\bX, \bY \in \BR^{n \times n}$ be any square matrices. Then
  \begin{equation}
    \label{eq:xy_ineq}
\bX \bX^\t \preceq 2 \bY \bY^\t + 2 \| \bX - \bY \|_\sigma^2 \cdot I.
\end{equation}
\end{lemma}
\begin{proof}(of Lemma \ref{lem:xy_ineq})
  Note that for any real numbers $x,y$ we have that $x^2 = (y + (x-y))^2 \leq 2y^2 + 2(x-y)^2$. It follows that for any vector $\bv \in \BR^n$,
  \begin{align*}
    \| \bX^\t \bv \|^2 &\leq 2\| \bY^\t \bv \|^2 + 2 \| \bX^\t \bv - \bY^\t \bv \|^2 \\
                       & \leq 2 \| \bY^\t \bv \|^2 + 2 \| \bX^\t - \bY^\t \|_\sigma^2 \| \bv \|^2,
  \end{align*}
  which establishes (\ref{eq:xy_ineq}).
\end{proof}

\begin{lemma}
  \label{lem:sr_ineq}
  Let $\bS, \bR \in \BR^{n \times n}$ be (symmetric) PSD matrices. Then
  \begin{equation}
    \label{eq:sr_ineq}
\bS \bR + \bR \bS \preceq 4 \bS^2 + 4 \| \bS - \bR \|_\sigma^2 \cdot I.
  \end{equation}
\end{lemma}
\begin{proof}(of Lemma \ref{lem:sr_ineq})
  Note that for any real numbers $r,s$ we have that $rs \leq 2s^2 + 2(r-s)^2$. It follows that for any $\bv \in \BR^n$,
  $$
2 \lng \bR \bv, \bS \bv \rng \leq 4 \| \bS \bv \|^2 + 4 \| \bR \bv - \bS \bv \|^2 \leq 4 \| \bS \bv \|^2 + 4 \| \bS - \bR \|_\sigma^2 \| \bv \|^2.
  $$
\end{proof}

\subsection{Proof of Lemma \ref{lem:ab_exist}}
\label{sec:ab_exist_proof}
\begin{proof}(of Lemma \ref{lem:ab_exist}).
  Since $F$ is continuously differentiable, by the fundamental theorem of calculus, for all $\bz$, 
  $$
F(\bz - \eta F(\bz)) = F(\bz) - \int_0^1 \partial F(\bz - (1-\alpha) \eta F(\bz)) \cdot \eta F(\bz) d\alpha,
$$
so if we set
$$
\bB_\bz = \int_0^1 \partial F(\bz - (1-\alpha) \eta F(\bz)) d\alpha,
$$
then we have $F(\bz - \eta F(\bz)) = F(\bz) - \eta \bB_\bz F(\bz)$. Again using the fundamental theorem of calculus,
$$
F(\bz - \eta F(\bz-  \eta F(\bz))) = F(\bz) - \eta \int_0^1 \partial F(\bz - (1-\alpha) \eta F(\bz - \eta F(\bz))) F(\bz - \eta F(\bz)) d\alpha.
$$
Then if we set
$$
\bA_\bz = \int_0^1 \partial F(\bz - (1-\alpha) \eta F(\bz - \eta F(\bz))) d\alpha,
$$
then
\begin{align*}
  F(\bz - \eta F(\bz - \eta F(\bz))) =& F(\bz) - \eta \bA_\bz F(\bz - \eta F(\bz))\\
  =& F(\bz) - \eta \bA_\bz (F(\bz) - \eta \bB_\bz F(\bz)) \\
  =& F(\bz) - \eta \bA_\bz F(\bz) + \eta^2 \bA_\bz \bB_\bz F(\bz).
\end{align*}
Note that $\bA_\bz, \bB_\bz$ have spectral norms at most $L$ and $\bA_\bz + \bA_\bz^\t$, $\bB_\bz + \bB_\bz^\t$ are PSD since the same is true of the matrices $\partial F(\bz - (1-\alpha) \eta F(\bz - \eta F(\bz)))$ and $\partial F(\bz - (1-\alpha) \eta F(\bz))$ (here we are using Lemma \ref{lem:nspsd}). Finally, since $F$ has a $\Lambda$-smooth Jacobian, we have that
\begin{align}
  \| \bA_\bz - \bB_\bz \|_\sigma &\leq \int_0^1 \| \partial F(\bz - (1-\alpha) \eta F(\bz)) - F(\bz - (1-\alpha) \eta F(\bz - \eta F(\bz))) \|_\sigma d\alpha\nonumber\\
                         & \leq \int_0^1 (1-\alpha) \eta \Lambda \| F(\bz) - F(\bz - \eta F(\bz)) \| d\alpha\nonumber\\
                         & = \frac{\eta \Lambda}{2} \| F(\bz) - F(\bz - \eta F(\bz))\|.\nonumber
\end{align}
\end{proof}

\subsection{Proof of Lemma \ref{lem:ab_diff}}
\label{sec:ab_diff_proof}
\begin{proof}(of Lemma \ref{lem:ab_diff}).
 Set $L_0 = \max \{ \| \bA \|_\sigma, \| \bB \|_\sigma \}$.  We wish to show that
  \begin{align*}
     (I - \bA + \bA \bB)(I - \bA + \bA \bB)^\t \preceq I \cdot (1 + 26 \| \bA - \bB \|_\sigma^2),
  \end{align*}
  or equivalently that
  $$
(\bA + \bA^\t) - (\bA \bB + \bB^\t \bA^\t) - \bA \bA^\t + (\bA \bB \bA^\t + \bA \bB^\t \bA^\t) - \bA \bB \bB^\t \bA^\t \succeq - 26 \| \bA - \bB \|_\sigma^2 I.
$$
Notice that $\bA \bB \bA^\t + \bA \bB^\t \bA^\t \succeq 0$ since for any vector $\bv \in \BR^n$, we have $\bv^\t \bA (\bB + \bB^\t) \bA^\t \bv \geq 0$ as $\bB + \bB^\t \succeq 0$. Moreover, since $\bB \bB^\t \preceq L_0^2 \cdot I$, we have that for any $\bv \in \BR^n$, $\bv^\t \bA \bB \bB^\t \bA^\t \bv \leq L_0^2 \cdot \bv^\t \bA \bA^\t \bv$, and so $\bA \bB \bB^\t \bA^\t \preceq L_0^2 \cdot \bA \bA^\t$. Thus it suffices to show
\begin{equation}
  \label{eq:ab_desired}
(\bA + \bA^\t) - (\bA \bB + \bB^\t \bA^\t) - (1 + L_0^2) \cdot \bA \bA^\t  \succeq - 26 \| \bA - \bB \|_\sigma^2 I.
\end{equation}
Next write $\bM := (\bA - \bA^\t)/2, \bS := (\bA + \bA^\t)/2, \bN := (\bB - \bB^\t)/2, \bR := (\bB + \bB^\t)/2$. Then $\bR, \bS$ are positive semi-definite and $\bM, \bN$ are anti-symmetric (i.e., $\bM^\t = -\bM, \bN^\t = - \bN$). Also note that $\| \bR - \bS \|_\sigma \leq \| \bA - \bB \|_\sigma$ and $\| \bM - \bN \|_\sigma \leq \| \bA - \bB \|_\sigma$. Then we have:
\begin{align*}
  \bA \bA^\t &= (\bM +\bS)(\bM^\t + \bS^\t) = \bM \bM^\t + \bM \bS + \bS \bM^\t + \bS \bS \\
  \bA \bB &= (\bM + \bS)(\bN + \bR) = \bM \bN + \bM \bR + \bS \bN + \bS \bR \\
  & \qquad \qquad \qquad \qquad \quad  = -\bM \bN^\t + \bM \bR - \bS \bN^\t + \bS \bR \\
  \bB^\t \bA^\t &= (\bN^\t + \bR^\t)(\bM^\t + \bS^\t) = \bN^\t \bM^\t + \bN^\t \bS + \bR \bM^\t + \bR \bS \\
  & \qquad \qquad \qquad \qquad \qquad \quad \  \  = -\bN \bM^\t - \bN \bS + \bR \bM^\t + \bR \bS.
\end{align*}
Next, note that for any vector $\bv \in \BR^n$ and any real number $\ep > 0$, we have
\begin{align*}
  \lng \bv, (\bM \bS + \bS \bM^\t) \bv \rng &= 2 \lng \bS \bv, \bM^\t \bv \rng \\
                                            & = 2 \sum_{j=1}^n (\bS \bv)_j \cdot (\bM^\t \bv)_j \\
  \text{(Young's inequality) } & \leq 2 \sum_{j=1}^n \frac{\ep\cdot (\bM^\t \bv)_j^2}{2} + \frac{(\bS \bv)_j^2}{2\ep} \\
                                            & = \ep \cdot \| \bM^\t \bv \|_2^2 + \frac{\| \bS \bv \|_2^2}{\ep}. 
\end{align*}
Thus $\bM \bS + \bS \bM^\t \preceq \ep \cdot \bM \bM^\t + \frac{\bS^2}{\ep}$. Replacing $\bM$ with $-\bN$ gives that for all $\ep > 0$, $-\bN \bS - \bS \bN^\t \preceq \ep \cdot \bN \bN^\t + \frac{\bS^2}{\ep}$, and replacing $\bS$ with $\bR$ gives that for all $\ep > 0$, $\bM \bR + \bR \bM^\t \preceq \ep \cdot \bM \bM^\t + \frac{\bR^2}{\ep}$. Hence
\begin{align*}
   (1 + L_0^2) \cdot \bA & \bA^\t + \bA \bB + \bB^\t \bA^\t \\
  & \preceq (1 + L_0^2) \cdot \left( (1 + \ep) \bM \bM^\t + (1 + \frac{1}{\ep}) \bS^2 \right) - \bM \bN^\t - \bN \bM^\t + \bS \bR + \bR \bS \\
  & \qquad \qquad \qquad \qquad \qquad \qquad + \ep \cdot \bN \bN^\t + \frac{\bS^2}{\ep} + \ep \cdot \bM \bM^\t + \frac{\bR^2}{\ep}.
\end{align*}
By Lemma \ref{lem:xy_ineq} with $\bX = \bR, \bY = \bS$ and Lemma \ref{lem:sr_ineq}, we have that
\begin{align*}
   \bS \bR + \bR \bS + \ep \cdot \bN \bN^\t & + \frac{\bS^2}{\ep} + \ep \cdot \bM \bM^\t + \frac{\bR^2}{\ep} \\
  & \preceq \ep \cdot (\bN \bN^\t  + \bM \bM^\t) +\left(4 +\frac{1}{\ep} + \frac{2}{\ep} \right)\cdot  \bS^2 + \left(4 + \frac 2\ep\right) \cdot \| \bR - \bS \|_\sigma^2 \cdot I, 
\end{align*}
so
\begin{align}
   (1 + L_0^2) \cdot \bA \bA^\t + \bA \bB & + \bB^\t \bA^\t \nonumber\\
  \label{eq:ab1}
  \preceq & \ep \cdot \bN \bN^\t + ((1 + L_0^2)(1 + \ep) + \ep) \bM \bM^\t - \bN \bM^\t - \bM \bN^\t \nonumber \\
   & \qquad + \left( (1 + L_0^2)\left(1+\frac 1\ep\right) + 4 + \frac{3}{\ep} \right) \bS^2 +  \left(4 + \frac 2\ep\right) \cdot \| \bA - \bB \|_\sigma^2 \cdot I.
\end{align}
Next, note that as long as $5\ep + 2L_0^2 + 2\ep L_0^2 \leq 1$, we have that
\begin{align}
  \ep \cdot \bN \bN^\t + ((1 + L_0^2)(1 & + \ep) + \ep) \bM \bM^\t - \bN \bM^\t  - \bM \bN^\t \nonumber\\
 \preceq^{(^*1)} & (\ep + 2 \cdot (2\ep + L_0^2 + \ep L_0^2)) \cdot \bN \bN^\t + \bM \bM^\t - \bN \bM^\t - \bM \bN^\t\nonumber\\
  & + 2 \cdot (2\ep + L_0^2 + \ep L_0^2) \cdot \| \bM - \bN \|_\sigma^2 I\nonumber\\
  \preceq & \bN \bN^\t + \bM \bM^\t - \bN \bM^\t - \bM \bN^\t + (1-\ep) \| \bA - \bB \|_\sigma^2 I \nonumber\\
  = & (\bN - \bM)(\bN^\t - \bM^\t) + (1-\ep) \| \bA - \bB \|_\sigma^2 I \nonumber\\
  \preceq & \| \bN - \bM \|_\sigma^2 I + (1-\ep) \| \bA - \bB \|_\sigma^2 I \nonumber\\
  \label{eq:ab2}
\preceq & (2-\ep) \| \bA - \bB \|_\sigma^2 I.
\end{align}
 where $(^*1)$ follows from Lemma \ref{lem:xy_ineq} with $\bX = \bM, \bY = \bN$. 
Moreover, as long as $$L_0 \left ((1 + L_0^2)(1 + (1/\ep)) + 4 + (3/\ep)\right) \leq 2,$$ since $\| \bS \|_\sigma \leq \| \bA \|_\sigma \leq L_0$, we have that
\begin{equation}
  \label{eq:ab3}
\left( (1 + L_0^2)\left(1+\frac 1\ep\right) + 4 + \frac{3}{\ep} \right) \bS^2 \preceq 2 \bS = \bA + \bA^\t.
\end{equation}
Combining (\ref{eq:ab1}), (\ref{eq:ab2}), and (\ref{eq:ab3}) gives that
$$
(1 + L_0^2) \cdot \bA \bA^\t + \bA \bB + \bB^\t \bA^\t \preceq \left(6 + \frac 2\ep \right) \| \bA - \bB \|_\sigma^2 I + \bA + \bA^\t,
$$
which is equivalent to (\ref{eq:ab_desired}) as long as $\ep = 1/10$. 

Finally, note that as long as $L_0 \leq 1/30$, the choice $\ep = 1/10$ satisfies $5\ep + 2L_0^2 + 2\ep L_0^2 \leq 1$ and $L_0 ((1 + L_0^2) (1 + 1/\ep) + 4 + 3/\ep) \leq 2$, completing the proof.

\end{proof}

\section{Proof of Last Iterate convergence of Proximal Point}
\label{sec:pp_ub}

We first prove the following lemma which shows that the Hamiltonian decreases each iteration of the proximal point algorithm:
\begin{lemma}
  \label{lem:approx_monotone}
  Suppose that $F : \BR^D \ra \BR^D$ is a monotone operator. Then 
  $$
\| F(\bx) \|^2  \leq \| F(\bx + \eta F(\bx)) \|^2.
$$
\end{lemma}
\begin{proof}
  By monotonicity of $F$ we have that, for $\eta > 0$, $\lng F(\bx), F(\bx + \eta F(\bx)) - F(\bx) \rng \geq 0$. Now note that
  \algnst{
    & \| F(\bx + \eta F(\bx))\|^2 - \| F(\bx) \|^2 \\
    & = 2 \lng F(\bx), F(\bx + \eta F(\bx)) - F(\bx) \rng + \| F(\bx + \eta F(\bx)) - F(\bx) \|^2 \\
    & \geq 0.
    }
  \end{proof}
  Theorem \ref{thm:pp_ub} gives an analogue of Theorem \ref{thm:eg_ub} for the PP algorithm. Given Lemma \ref{lem:approx_monotone}, its proof is essentially immediate given prior results in the literature (see, e.g., \cite{mokhtari_convergence_2019,monteiro_complexity_2010}), but we reproduce the entire proof for completeness.
  \begin{theorem}
    \label{thm:pp_ub}
    Suppose $F : \BR^n \ra \BR^n$ is a monotone operator. Fix some $\bz\^0 \in \BR^n$, and suppose there is $\bz^* \in \BR^n$ so that $F(\bz^*) = 0$ and $\| \bz^* - \bz\^0 \| \leq D$. 
    If the proximal point algorithm with any step size $\eta > 0$ is initialized at $\bz\^0$, then its iterates $\bz\^T$ satisfy
  $$
\| F(\bz\^t) \| \leq \frac{D}{\eta \sqrt{T}}.
$$
If moreover $\MZ = \MB(\bx^*, D) \times \MB(\by^*, D)$ and $F(\bx, \by) = \matx{ \nabla_\bx f(\bx,\by) \\ - \nabla_\by f(\bx,\by) }$ for a convex-concave function $f$, then it follows that
$$
\Gap_f^\MZ(\bx\^T, \by\^T) = \max_{\by' \in \MB(\by^*, D)} f(\bx\^T, \by') - \min_{\bx' \in \MB(\bx^*, D)} f(\bx', \by\^T) \leq \frac{\sqrt{2}D^2}{\eta \sqrt{T}}.
$$
(Here $\bz\^T = (\bx\^T, \by\^T)$.)
\end{theorem}
\begin{proof}
\noindent Recall that the iterates of the proximal point algorithm are defined by
  $$
\bz\^{t+1} = \bz\^t - \eta F(\bz\^{t+1}).
$$
  
  It is easy to see that the following equality holds at all iterations of the proximal point algorithm: for all $\bz \in \BR^D$, 
  $$
\lng F(\bz\^{t+1}), \bz\^{t+1} - \bz \rng =  \frac{1}{2\eta} \left( \| \bz\^t - \bz \|^2 - \| \bz\^{t+1} - \bz \|^2 - \| \bz\^t - \bz\^{t+1} \|^2 \right).
$$
Setting $\bz = \bz^*$, so that $\lng F(\bz'), \bz' - \bz^* \rng \geq 0$ for all $\bz'$, it follows that for any $T > 0$,
$$
\sum_{t = 0}^{T-1} \frac{\eta}{2} \| F(\bz\^{t+1}) \|^2 \leq \sum_{t=0}^{T-1} \frac{1}{2\eta} \left( \| \bz\^t - \bz \|^2 - \| \bz\^{t+1} - \bz \|^2\right) \leq \frac{1}{2\eta} \| \bz_0 - \bz\|^2 \leq \frac{1}{2\eta} D^2.
$$
(The last inequality follows since $\bz, \bz^* \in \MZ$, and the diameter of $\MZ$ is at most $D$.)
Thus, there exists some $t^* \in \{1, 2, \ldots, T \}$ so that
\begin{equation}
  \label{eq:min_hamilt}
 \| F(\bz\^{t^*}) \|^2 \leq \frac{D^2}{\eta^2 T}.
 \end{equation}
Next, Lemma \ref{lem:approx_monotone} with $\bx = \bz\^{t+1}$ gives that for each $t \geq 0$,
 $$
\| F(\bz\^{t+1} ) \|^2 \leq \| F(\bz\^{t+1} + \eta \bz\^{t+1})\|^2 \leq \| F(\bz\^t) \|^2.
$$
Thus
$$
\| F(\bz\^T) \| \leq \frac{D}{\eta \sqrt{T}}.
$$
If $\MZ = \MX \times \MY$ and $F(\bx,\by) =  \matx{ \nabla_\bx f(\bx,\by) \\ -\nabla_\by f(\bx,\by)} $, for some convex-concave function $f$, then
\begin{align*}
  \max_{\by' \in \MY} f(\bx\^T, \by') &- \min_{\bx' \in \MX} f(\bx', \by\^T) \nonumber \\
  & = \max_{\by' \in \MY} f(\bx\^T, \by') - f(\bx\^T, \by\^T) - \min_{\bx' \in \MX} (f(\bx', \by\^T) - f(\bx\^T, \by\^T)) \\
  & \leq \max_{\by' \in \MY} \lng \nabla_\by f(\bx\^T, \by\^T), \by' - \by\^T \rng + \max_{\bx' \in \MX} \lng \nabla_\bx f(\bx\^T, \by\^T), \bx\^T - \bx' \rng \\
  & = \max_{z' \in \MZ} \lng F(\bz\^T), \bz\^T - \bz' \rng \\
  & \leq \| F(\bz\^T) \| \cdot D\sqrt{2} \\
  & \leq \frac{\sqrt{2} D^2}{\eta \sqrt{T}}.
\end{align*}
\end{proof}

\section{CLIs with time-varying coefficients}
\subsection{Averaged EG iterates are a 2-CLI with time-varying coefficients}
\label{sec:2scli-averaged}
In this section we show that the averaged iterates of extra-gradient can be written as the iterates of a particular 2-CLI\footnote{See \cite[Definition 2]{arjevani_iteration_2016} for a definition of 2-CLIs.} with time-varying coefficients. Let $\bz\^t$, $t \geq 0$ be the iterates of extragradient and $\bv\^t := \frac{\sum_{t'=0}^t \bz\^t}{t+1}$. Then $\bz\^t = (t+1) \bv\^t - t \bv\^{t-1}$, and so
\begin{align*}
  (t+2) \bv\^{t+1} &= (t+1) \bv\^t + \bz\^{t+1} \\
                   &= (t+1) \cdot \bv\^t + \bz\^t - \eta F(\bz\^t - \eta F(\bz\^t)) \\
                   &= 2(t+1) \bv\^t - t\bv\^{t-1} - \eta F((t+1) \bv\^t - t\bv\^{t-1} - \eta F((t+1)\bv\^t - t\bv\^{t-1})).
\end{align*}
Writing $F(\bz) = \bA \bz + \bb$ (i.e., restricting to the setting in which CLIs are defined) gives
\begin{align*}
\bv\^{t+1} = \left(2I - \eta \bA + (\eta\bA)^2\right) \cdot \frac{t+1}{t+2} \bv\^t - \left(I - \eta \bA + (\eta \bA)^2\right) \frac{t}{t+2} \bv\^{t-1} + \eta(-I + \eta \bA)\bb,
\end{align*}
which is a 2-CLI.

\subsection{Lower bound for last iterate of EG with time-varying coefficients}
\label{sec:eg-timevarying}
In this section we consider the iterates of the EG algorithm with time-varying coefficients, defined as follows:
\begin{align}
  \label{eq:eg-iterates-timevarying}
\bbz^{(t + 1/2)} = \bbz^{(t)} - \eta_t F(\bbz^{(t)}) \qquad \qquad \bbz^{(t + 1)} &= \bbz^{(t)} - \eta_t F(\bbz^{(t+1/2)}),
\end{align}
where $\eta_t > 0$ is a sequence of time-varying coefficients.

As in Theorem \ref{thm:1scli_lb}, for given $n, L, D$ and $F = F_f \in \MFbil_{n,L,D}$ with Nash equilibrium $\bz^* = (\bx^*, \by^*)$, set $\MZ = \MB(\bx^*, D) \times \MB(\by^*, D)$, so that for $\bz = (\bx, \by)$, $\Gap_f^\MZ(\bz) = \sup_{\by' : \| \by' - \by^*\| \leq D} f(\bx, \by') - \inf_{\bx' : \| \bx' - \bx^* \| \leq D} f(\bx', \by)$.
\begin{proposition}
  Fix any $n \in \BN$ and $L, D > 0$. For $t \in \BN$, let $\eta_t$ be a sequence with $\eta_t \in (0,1/L)$ for each $t$. Let $\bz\^t = (\bx\^t, \by\^t)$ denote the iterates of EG with step-sizes $\eta_t$, as in (\ref{eq:eg-iterates-timevarying}). Then there is some $F = F_f \in \MFbil_{n,L,D}$ such that $\Gap_f^\MZ(\bz\^T) \geq \frac{LD^2}{4\sqrt{T}}$. 
\end{proposition}
\begin{proof}
  The proof closely parallels that of Theorem \ref{thm:1scli_lb}. In particular, we choose $\bA, \bb$ as in (\ref{eq:choose_A}), with $\nu \in (0,L)$ to be specified below. Set $\bC_{0,t}(\bA) := I - \eta_t \bA + (\eta_t \bA)^2$. Then the update (\ref{eq:eg-iterates-timevarying}) can be written as
  $$
\bz\^t = \bC_{0,t}(\bA) \cdot \bz\^{t-1} - (I - \eta_t \bA) \cdot \eta_t \bb,
  $$
  We claim that for $t \geq 0$, 
  \begin{equation}
    \label{eq:hamt-timevarying}
\bA \bz\^t + \bb = \prod_{t'=0}^{t-1} \bC_{0,t'}(\bA) \cdot  \bb.
\end{equation}
To see (\ref{eq:hamt-timevarying}), we argue by induction, noting that the base case is immediate since $\bz\^0 = \bbzero$, and for the inductive step:
\begin{align}
  \bA \bz\^{t+1} + \bb&=\bA \bC_{0,t}(\bA) \bz\^{t-1}  - \bA(I - \eta_t \bA) \eta_t \bb + \bb\nonumber\\
                      &= \bA \bC_{0,t}(\bA) \bz\^{t-1}  + (\bC_{0,t}(\bA) - I) \bb + \bb\nonumber\\
                      &=\bC_{0,t}(\bA) \cdot (\bA \bz\^{t-1} + \bb).\nonumber
\end{align}
Thus, by (\ref{eq:gap_ratio}), we have
\begin{equation}
    \label{eq:gap-timevarying}
\Gap_f^\MZ(\bx\^t, \by\^t) = \frac{ D^2 \nu \left\| \prod_{t'=0}^{t-1} \bC_{0,t'}(\bA) \cdot \bb \right\|}{\| \bb \|}.
\end{equation}
Let us choose $\nu = L/\sqrt{T}$. It is straightforward to check that the singular values of $\prod_{t'=0}^{t-1} \bC_{0,t'}(\bA)$, which are equal to the magnitudes of its eigenvalues, are all equal to
\begin{align*}
  \left| \prod_{t'=0}^{t-1} (1 - \eta_t \nu i + (\eta_t \nu i)^2) \right| &= \prod_{t'=0}^{t-1} \left|(1 - \eta_t^2 \nu^2) + \eta_t \nu i \right| \\
                                                                          & \geq \prod_{t'=0}^{t-1} (1 - \eta_t^2 \nu^2) \geq (1 - 1/T)^T \geq 1/4
\end{align*}
for $T \geq 2$. For any nonzero choice of $\bb$, (\ref{eq:gap-timevarying}) thus gives that $\Gap_f^\MZ(\bz\^t) \geq \frac{D^2 L}{4\sqrt{T}}$. 
\end{proof}
\end{document}